\documentclass[twocolumn]{svjour3}   
\smartqed  
\usepackage{booktabs}
\usepackage[ruled]{algorithm2e}
\usepackage{mathptmx}
\usepackage{graphicx}
\usepackage{times}
\usepackage{epsfig}
\usepackage{epstopdf}
\usepackage{amsmath}
\usepackage{amssymb}
\usepackage{subfigure}
\usepackage{caption}
\usepackage{tabularx}
\usepackage[table]{xcolor}
\usepackage[normalem]{ulem}
\usepackage[noend]{algpseudocode}
\usepackage{booktabs}
\usepackage{multirow}
\usepackage{lipsum}
\usepackage{float}
\usepackage[cmintegrals]{newtxmath}
\usepackage{graphicx}
\usepackage{color}
\usepackage[colorlinks]{hyperref}

\makeatletter
\DeclareRobustCommand\onedot{\futurelet\@let@token\@onedot}
\def\@onedot{\ifx\@let@token.\else.\null\fi\xspace}
 
\def\ie{\emph{i.e}\onedot}

\def\etal{\emph{et al}\onedot}

\makeatother

\newcolumntype{L}[1]{>{\raggedright\arraybackslash}p{#1}}
\newcolumntype{C}[1]{>{\centering\arraybackslash}p{#1}}
\newcolumntype{R}[1]{>{\raggedleft\arraybackslash}p{#1}}

\begin{document}

\title{A Causal Inspired Early-Branching Structure \\ for Domain Generalization}

\author{$\text{Liang~Chen}^{1}$ \and
        $\text{Yong~Zhang}^{2\star}$ \and
        $\text{Yibing~Song}^{3}$ \and
        $\text{Zhen~Zhang}^{1}$ \and
        $\text{Lingqiao~Liu}^{1\star}$ \\
        \email{\{liangchen527,~~zhangyong201303,~~yibingsong.cv\}@gmail.com},\\
        ~~zhen@zzhang.org,~~lingqiao.liu@adelaide.edu.au}

\authorrunning{Liang Chen et al.} 

\institute{$~^1$ The University of Adelaide, SA, Australia. \\
$~^2$ Tencent AI Lab, Shenzhen, China.\\
$~^3$ Alibaba DAMO Academy, Hangzhou, China. \\
$\star$ Corresponding authors. \\
$\dag$ This work is done when L. Chen interned at Tencent AI Lab. Code is available at \url{https://github.com/liangchen527/CausEB}.} 

\date{Received: date / Accepted: date}

\maketitle

\begin{abstract}
Learning domain-invariant semantic representations is crucial for achieving domain generalization (DG), where a model is required to perform well on unseen target domains. One critical challenge is that standard training often results in entangled semantic and domain-specific features. Previous works suggest formulating the problem from a causal perspective and solving the entanglement problem by enforcing marginal independence between the causal (\ie semantic) and non-causal (\ie domain-specific) features. Despite its simplicity, the basic marginal independent-based idea alone may be insufficient to identify the causal feature. By d-separation, we observe that the causal feature can be further characterized by being independent of the domain conditioned on the object, and we propose the following two strategies as complements for the basic framework.

First, the observation implicitly implies that for the same object, the causal feature should not be associated with the non-causal feature, revealing that the common practice of obtaining the two features with a shared base feature extractor and two lightweight prediction heads might be inappropriate. To meet the constraint, we propose a simple early-branching structure, where the causal and non-causal feature obtaining branches share the first few blocks while diverging thereafter, for better structure design;
Second, the observation implies that the causal feature remains invariant across different domains for the same object. To this end, we suggest that augmentation should be incorporated into the framework to better characterize the causal feature, and we further suggest an effective random domain sampling scheme to fulfill the task. 
Theoretical and experimental results show that the two strategies are beneficial for the basic marginal independent-based framework.

\end{abstract}

\maketitle

\newcommand{\red}[0]{\textcolor{red}}
\newcommand{\blue}[0]{\textcolor{blue}}
\newcommand{\gray}[0]{\textcolor{gray}}
\newcommand{\rot}[0]{\rotatebox{90}}
\newtheorem{prop}{Proposition}

\section{Introduction}
Deep learning achieves tremendous success in many fields with useful real-world applications \cite{silver2016mastering,jumper2021highly}. However, most current deep models are developed based on the assumption that data from training and testing environments follows the same distribution, which does not always hold in practice. This problem can result in poor performance for delicately designed models and hinders further deep learning applications. To this end, it is important to develop domain generalization (DG) methods that
can maintain good performance in arbitrary environments. 

The key of DG is to learn domain-invariant semantic representations that are robust to domain shift \cite{ben2006analysis}. However, standard training often results in entangled semantic and domain-specific features, which may change drastically in a new domain, hindering the model from generalizing to new environments. To address the issue, existing methods introduce various forms of regularization, such as adopting alignment \cite{muandet2013domain,ghifary2016scatter,li2018domain,hu2020domain,cha2022domain},
using domain-adversarial training \cite{ganin2016domain,li2018domain,yang2021adversarial,li2018deep}, or developing meta-learning methods \cite{li2018learning,balaji2018metareg,dou2019domain,li2019episodic}. Despite the success of these arts on certain occasions, DG remains challenging and is far from being solved. For example, as a recent study \cite{gulrajani2020search} suggests, under a rigorous evaluation protocol, it turns out that the naive empirical risk minimization (ERM) method \cite{vapnik1999nature}, which aggregates training data from all domains and trains them in an end-to-end manner without additional efforts, can perform competitively against most elaborate alternatives. This observation highlights the need for a more effective method that is capable of obtaining the domain-invariant semantic feature.

To derive a better objective to obtain the domain-invariant representation, following prior arts \cite{heinze2021conditional,atzmon2020causal,mahajan2021domain}, we model the image generation process with a common structural causal model (SCM). As depicted in Fig.~\ref{fig scm}, images are considered to be generated by interventions \footnote{In causal relation, denoted as ``cause $\rightarrow$ effect", intervention means changing the value of a variable without affecting its causes but affecting its effects (variables that depend on it) \cite{pearl2009causality,peters2017elements}.} on the ``Object (O)" and ``Domain (D)", which are the causes of the domain-invariant semantic (causal: $X_o$) and domain-specific (non-causal: $X_d$) features, and only the semantic feature causes the target output. We cast the DG problem of locating the intervention that causes the semantic label of given images.  
In the general case, the entanglement of causal and non-causal features will result in an undesired situation where the non-causal information is also encoded in the target classifier, inevitably corrupting the prediction when an unexpected intervention is imposed on D during inference.

To explicitly disentangle the two types of features, previous arts~\cite{ganin2016domain,albuquerque2019generalizing,atzmon2020causal,chen2021style} often build their methods upon the marginal independence assumption, where the causal feature is considered independent of the non-causal one (\ie $X_o \Perp X_d$), and the concept is often implemented by leveraging a classical dual-branching network where each type of feature is extracted through separate branches. We follow this classical design and employ Hilbert-Schmidt Information Criterion (HSIC) \cite{gretton2005measuring,gretton2007kernel} as a measurement for the feature independence. A glimpse of this classical dual-branch design is presented in Fig~\ref{fig pipeline}.

However, despite its simplicity, there is a major concern regarding this common design: the independence constraint alone may be insufficient to determine the optimal causal feature due to the ill-posed nature of this problem~\cite{atzmon2020causal}. To ease this concern, we follow suggestions from existing arts~\cite{atzmon2020causal,halva2020hidden,hyvarinen2016unsupervised} by enforcing more constraints on the causal feature. By the d-separation criterion, we observe that the causal feature $X_o$ further satisfies $X_o \Perp D \vert O$.   
Based on this observation, we propose to use the following two strategies as complements for the classical marginal independent-based framework.

\begin{figure}
\centering
	\begin{minipage}[b]{0.78\linewidth}
		\centering
		\centerline{
			\includegraphics[width =\linewidth]{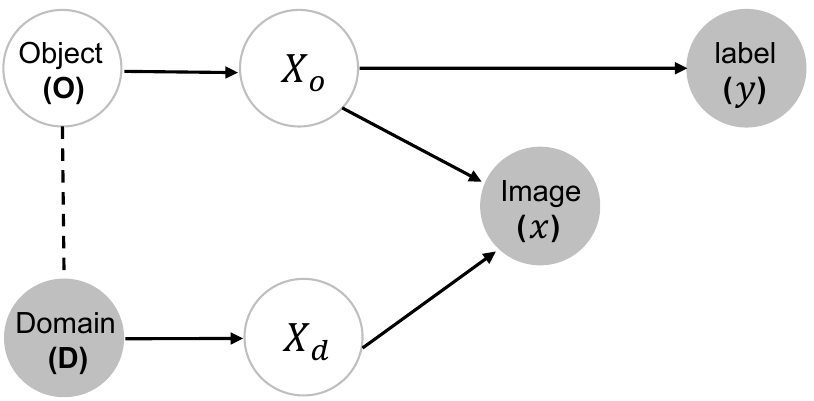}}
	\end{minipage}
	\caption{Structural causal models for the image generation process. Observed variables are shaded. Solid arrows represent causal relations. The dashed line denotes there is entanglement between the two variables. We consider images to be generated by interventions on the ``Object" and ``Domain" variables which causes the causal semantic features $X_o$ and non-causal domain-specific features $X_d$.}
	\vspace{-0.3 cm}
	\label{fig scm}
\end{figure}

\textbf{First, $X_o \Perp D \vert O$ implicitly implies that there should be no association between causal and non-causal features from the same object.} This idea can provide constructive guidance in the corresponding network structure design: the representation extractor utilized for the causal feature should remain independent of the one for the non-causal feature. To this end, we argue that the common practice, which uses a shared base feature extractor with two lightweight prediction heads in the classical dual-branching structure~\cite{ganin2016domain,chen2021style,atzmon2020causal}, will be detrimental to the performance as the causal feature extractor will rely heavily on the non-causal feature with this setting. Instead, we suggest an early-branching architecture, where the two branches share the first few blocks while diverging thereafter, for the dual-branch network. We further provide theoretical and experimental evidence to support this argument in Sec.~\ref{sec complets} and~\ref{sec ana}.

\textbf{Second, $X_o \Perp D \vert O$ implies that the causal feature remains invariant across different domains for the same object.} However, directly implementing this idea is non-trivial: object information is often unavailable in existing datasets. Instead of pairing samples heuristically based on their feature distance~\cite{atzmon2020causal}, we treat each image as a distinct object and employ augmentation to simulate samples from diverse domains that correspond to the same object. 
To generate unlimited and diverse domains, we further introduce a new random domain sampling (RDS) approach that expands the domain types by incorporating features with randomly altered styles. Unlike previous methods \cite{zhou2021domain,li2022uncertainty} that synthesize new domains by mixing or adding noise, RDS can synthesize diverse domains by directly perturbing the styles of the feature maps with adjustable intensity.  

Derived from the conditional independent constraint between the causal feature and domain, the main contributions of this work can be summarized in twofold,
\begin{itemize}
    \item We theoretically prove that the common network structure design, which obtains the causal and non-causal features with a shared base feature extractor and two lightweight prediction heads, should be detrimental to the performance. This work argues that an early-branching structure should be imposed on the classical dual-branch pipeline, which is further validated by experiments in different network backbones.
    \item We suggest that the augmentation idea can be utilized to implement the conditional independent constraint, and we introduce a random domain sampling strategy for augmentation, which can synthesize unlimited diverse domains.
\end{itemize}

Through our experimental studies in current benchmarks~\cite{gulrajani2020search,koh2021wilds}, we illustrate that the proposed two strategies are complementary for the independent-based framework, and our complete method can perform favorably against existing state-of-the-art models.

\section{Related Works}
Various approaches have been proposed to tackle the distribution shift problem lately\cite{muandet2013domain,shi2021gradient,pandey2021generalization,rame2021ishr,harary2022unsupervised,zhou2021domain,li2022uncertainty,xu2021fourier,li2021simple,li2018domain,li2018deep,li2018learning,balaji2018metareg,dou2019domain,li2019episodic,chen2022mix,chen2022compound,chen2022self,chen2022ost,chen2023improved,chen2023domain,wang2022out,wang2022causal}. This section briefly reviews works that are highly associated with ours.

\noindent\textbf{Invariant representation learning.}
The pioneering work \cite{ben2006analysis} theoretically proved that if the features remain invariant across different domains, then they are general and transferable to different domains. Inspired by this theory, many arts aim to use deep networks to explore domain-invariant features. 
For example, \cite{ganin2016domain} trains a domain-adversarial neural network to obtain domain-invariant features by maximizing the domain classification loss. This idea is further explored by \cite{li2018domain}. They employ a maximum mean discrepancy constraint for the representation learning of an auto-encoder via adversarial training. Instead of directly obtaining the semantic features, some arts \cite{khosla2012undoing,li2017deeper} suggest decomposing the model parameters into domain-invariant and domain-specific parts and only using the domain-invariant parameters for prediction. In \cite{cha2022domain}, the learned features are enforced to be similar to that from the pretrained model in ensuring invariance. This idea is further improved by \cite{chen2023domain}. The rational matrices, which represent the mixed results of features and classifier weights, are enforced to be similar for samples belonging to the same category. Recently, the task has been further explored at a gradient level. \cite{koyama2020invariance} learn domain-invariant features by minimizing the variances of inter-domain gradients. Inspired by the fact that optimization directions should be similar across domains, \cite{shi2021gradient} maximize the gradient inner products between domains to maintain invariance. 

\begin{figure}
\centering
	\begin{minipage}[b]{0.9\linewidth}
		\centering
		\centerline{
			\includegraphics[width =\linewidth]{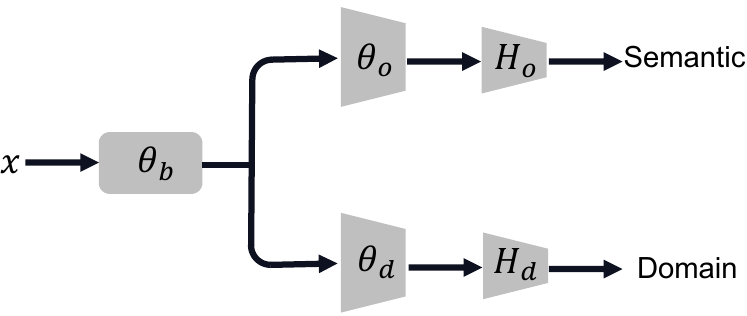}}
	\end{minipage}
	\caption{A glimpse of the classical dual-branch network for modeling independence between causal and non-causal features \cite{ganin2016domain,albuquerque2019generalizing,atzmon2020causal,chen2021style}, where the two branches share a same base feature extractor (\ie $\theta_b$) with two light-weight prediction heads. Here the upper branch is used to extract the causal features (\ie $F_o(x) = \theta_b (\theta_o (x))$), and the lower branch for extracting the non-causal features (\ie $F_d(x) = \theta_b (\theta_d (x))$). The independent constraint is enforced between $F_o(x)$ and $F_d(x)$. $H_o$ and $H_d$ are classifiers for the corresponding branches. Only the upper branch is utilized during inference.}
	\label{fig pipeline}
 \vspace{-0.4 cm}
\end{figure}

\noindent\textbf{Augmentation.}
Besides the typical augmentation strategies, such as rotation, flipping, and cropping, there are some methods that are specially designed for the DG task. For example, \cite{yan2020improve,chen2022mix} apply the mixup idea \cite{zhang2017mixup} for DG by directly mixing images and labels from a batch. Instead of mixing in the image level, \cite{kim2021selfreg} proposes to mix the feature representations, \cite{xu2021fourier} suggests linearly interpolating the amplitude spectrums of different samples. How to synthesize new samples for training is also studied in the literature: generative modeling is used in \cite{zhou2020learning,carlucci2019hallucinating} to create new domains and domain-agnostic images. Inspired by the success in the related area, \cite{zhou2021domain,nam2021reducing} use AdaIN \cite{huang2017arbitrary} to synthesize new domains by mixing the styles of the features. To further explore the idea, \cite{li2022uncertainty} propose to add perturbations to the styles based on their variances, and \cite{kang2022style} suggests storing seen styles in a queue and consistently generating distinct styles for mixing.
Despite certain improvements, their augmented domains may still suffer from homogeneity. To obtain unlimited and diverse domains, we assume the styles of the features to follow normal distributions. Consequently, we can sample unlimited diverse domain information that derives from the original styles. By integrating the augmented features in an effective framework, our method performs favorably against existing arts in the DomainBed benchmark \cite{gulrajani2020search}.

\noindent\textbf{Causality-inspired DG.}
Using causality to describe the generalization problem has been studied in some previous methods~\cite{peters2016causal,christiansen2021causal,scholkopf2012causal,zhang2013domain}. The pioneer art~\cite{gong2016domain} assumes the images are generated by the ``label" variable, and they further use the generation graph to learn invariant components that are transferable for different domains. The similar generation graph is further utilized and extended in~\cite{magliacane2018domain,rojas2018invariant,heinze2021conditional}. Meanwhile, there are other works on connection causality~\cite{arjovsky2019invariant,peters2016causal} that consider the opposite where ``image" is the cause and ``label" is the effect. The causal model used in our work can be traced in~\cite{mahajan2021domain,atzmon2020causal,liu2021learning} where ``object" and ``domain" are both considered causes for the image. Different from previous approaches that use generative models~\cite{atzmon2020causal,liu2021learning} or feature matching~\cite{mahajan2021domain} to localize the causal features, we improve the localization with a more reasonable early-branching structure and an effective augmentation scheme, which are both inspired by the conditional invariant finding.
Some recent studies also focus on using causality for DG. For example, \cite{wang2022causal} uses a balance score to sample mini-batches, which is expected to be spurious-free and can produce minimax optimal to identify causal features. Instead of specifically identifying causal features among different domains, \cite{wang2022out} theoretically reveals that using a subset of causal invariant transformations to modify the input data can also achieve minimax optimality across all the domains of interest.

\section{Basic Framework}
\subsection{Preliminary} 
Let $\mathcal{X}$ denote the image space, and $\mathcal{Y}$ represent the label space. In the vanilla domain generalization (DG) problem, we are often given $M$ source domains, denoting as $\mathcal{D}_s = \{\mathcal{D}_1,\mathcal{D}_2,\cdots,\mathcal{D}_M \}$ that sampled from different probability distributions on the joint space $\mathcal{X} \times \mathcal{Y}$. Then, each $\mathcal{D}_m \in \mathcal{D}_s$ can be represented as $\mathcal{D}_m = \{(x_i^m, y_i^m, d_i^m) \}_{i=1}^{N_m}$, where $d^m$ is the domain label and $N_m$ is the number of data label pairs in the $m^{th}$ domain. The task of vanilla DG is to learn a model from $\mathcal{D}_s$ for making predictions on an unseen $\mathcal{D}_{M+1}$ domain. Denoting the feature space as $\mathcal{C}$, the baseline ERM method can be regarded as finding functions $F: \mathcal{X} \rightarrow \mathcal{C}$ and $H: \mathcal{C} \rightarrow \mathcal{Y}$ such that $H(F(x))$ leads to optimal results for $\forall x \in \mathcal{D}_{M+1}$.

\subsection{Marginal Independent-Based Disentanglement}
\label{sec frame}
Same as existing arts~\cite{ganin2016domain,albuquerque2019generalizing,atzmon2020causal,chen2021style}, our basic framework aims to ease disentanglement using the classical marginal independence assumption where the causal and non-causal features are considered independent of each other, corresponding to $X_o \Perp X_d$ by d-separation in the SCM from Fig.~\ref{fig scm}. The idea is fulfilled by explicitly enforcing independent regularizations upon $X_o$ and $X_d$.
Since both $X_o$ and $X_d$ are unobserved, we follow the existing practice by learning them separately through representation functions $F_o$ and $F_d$: $\mathcal{X} \rightarrow \mathcal{C}$.

Different from previous methods, we use Hilbert-Schmidt Information Criterion (HSIC) \cite{gretton2005measuring,gretton2007kernel} as the independence measure to regularize the statistic dependency between features from $F_o$ and $F_d$, and the ideal loss-minimizing function $F_o^{\ast}$ that achieves disentanglement can be written as,
\begin{equation}
\begin{aligned}
\label{eq ind}
&\arg \min_{F_o(x)} \text{HSIC} (F_o(x), F_d(x)) \\
= & \arg \min_{F_o(x)} \frac{1}{(b-1)^2}\textbf{tr}(\textbf{KHLH}),
\end{aligned}
\end{equation}
where $b$ is the batch size; $\textbf{H} = \textbf{I} - \frac{1}{b}$ is a centering matrix; $\textbf{tr}(\cdot)$ computes the trace of the given matric;{following the pratice in \cite{atzmon2020causal}, we use linear kernels for computing HSIC}: $\textbf{K} = F_o(x)' F_o(x)'^{\text{T}}$; $\textbf{L} = F_d(x)' F_d(x)'^{\text{T}}$, given that $F_o(x)' = \frac{F_o(x)}{\Vert F_o(x) \Vert }$, and similar for $F_d(x)'$. Notably, different from using Gaussian kernel where $\textbf{K}_{a, b} = \exp (- \frac{( F_o(x)'_a - F_o(x)'_b) ^2}{\sigma^2})$ ($a,b$ are the entry indexes and $\sigma$ is the variance), using linear kernels does not involve kernel and parameter selection, avoiding quadratic computational complexity, thus can be applied to process large-size features.

Note the optimal objective is to find the ideal $F_o^{\ast}$ to obtain the target causal feature that is used for predicting the semantical class. Thus the independent constraint will not be utilized to find $F_d^{\ast}$. Meanwhile, one may argue that there are other independent constraints, such as orthogonal and correlation minimization. We show in Sec.~\ref{sec ana} that we adopt HSIC because it is more effective than other alternatives. 
Following the prevalent designs~\cite{ganin2016domain,albuquerque2019generalizing,atzmon2020causal,chen2021style}, we implement $F_o$ and $F_d$ by utilizing a dual-branch network with a shared base feature extractor ($\theta_b \triangleq F_o \cap F_d$ in Fig.~\ref{fig pipeline}) and two branches ($\theta_o \triangleq F_o \backslash (F_o \cap F_d)$ and $\theta_d \triangleq F_d \backslash (F_o \cap F_d)$ in Fig.~\ref{fig pipeline}) processing different representations. 

\begin{figure}
\centering
	\begin{minipage}[b]{0.49\linewidth}
		\centering
		\centerline{
			\includegraphics[width =\linewidth]{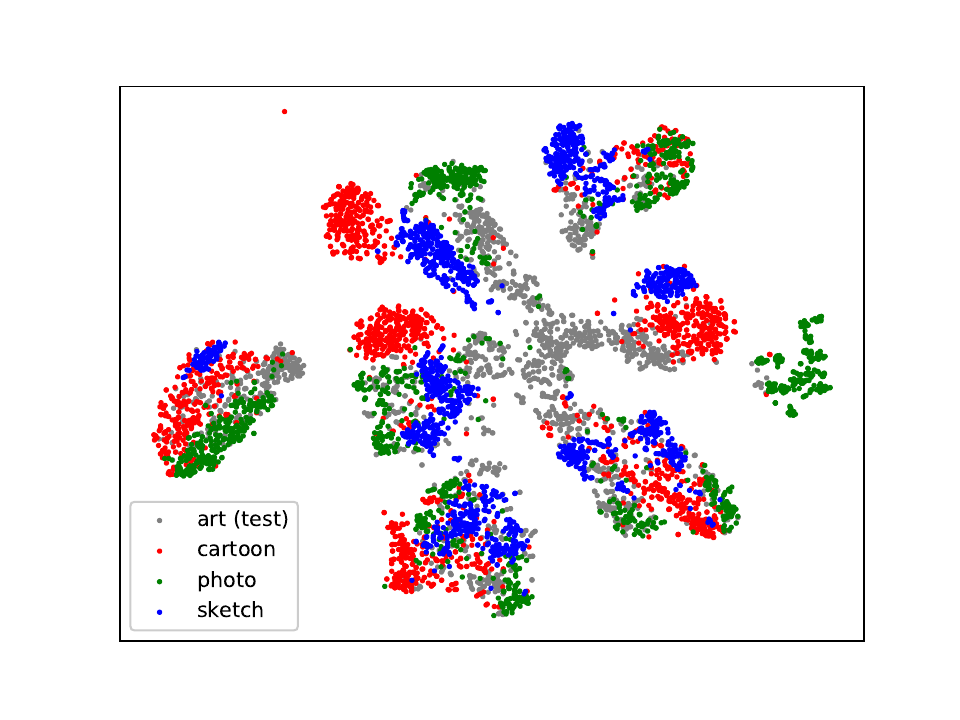}}
			\centerline{\small{(a) ERM}}
	\end{minipage}
	\begin{minipage}[b]{0.49\linewidth}
		\centering
		\centerline{
			\includegraphics[width =\linewidth]{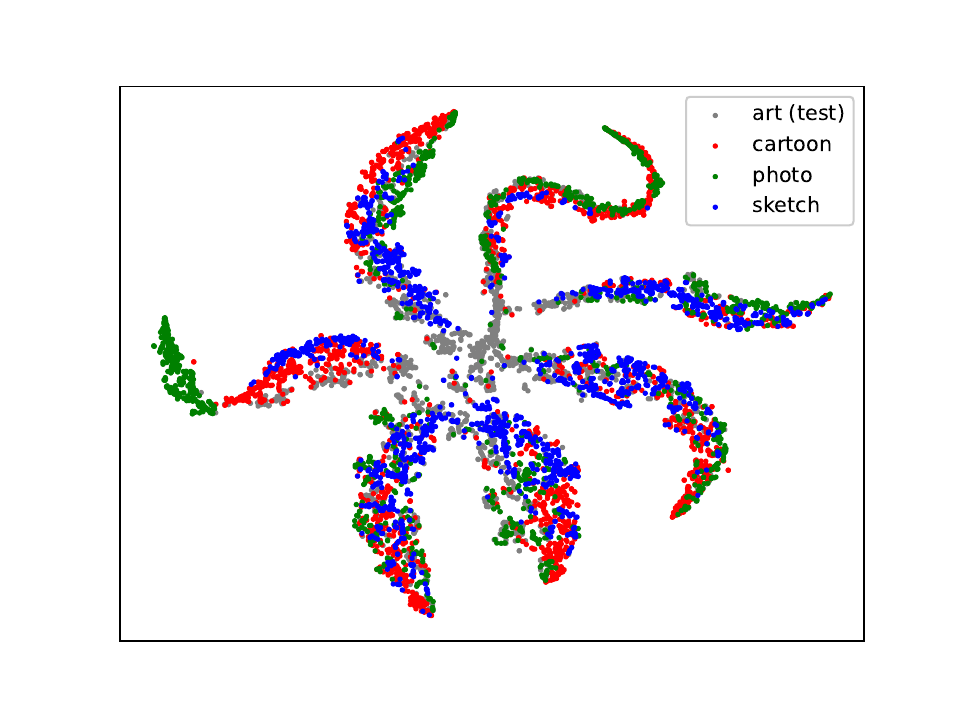}}
			\centerline{\small{(b) Independent-based framework}}
	\end{minipage}
	\caption{2D t-SNE visualizations of target semantic representations from the ERM model and our basic form. The PACS dataset \cite{li2017deeper} is used with art as the unseen target domain. The seven clusters in (a) and (b) denote the corresponding classes. The domain information is more obvious in the target features from the ERM model, indicating that ERM tends to learn entangled domain-specific and semantic features. In comparison, using the independent constraint can better disentangle the two features, resulting in less domain information in the target features.}
	\label{fig tsne}
 \vspace{-0.4 cm}
\end{figure}

Visualized target causal features from the baseline ERM method and the adopted independent-based framework are plotted in Fig. \ref{fig tsne} (a) and (b). We note that although both methods can separate different classes, our target features tend to include less domain information as feature points from different domains are mixed. In comparison, we observe that target features from different domains are clearly separated in ERM, indicating that ERM encodes more non-causal domain-specific information in their target features. These results validate that the basic independent-based regularization term can help disentangle the causal and non-causal features.

\section{Methodology}
\subsection{Complements for the Basic Framework}
\label{sec complets}
As stated in Proposition 2 in \cite{atzmon2020causal}, identifying the causal features is non-trivial as multiple $X_o$ can lead to the same observed data (\ie $x, y, d$). This property holds even with the independent constraint $X_o \Perp X_d$. A reasonable approximation for this dilemma is to use more constraints to find an invariant that can characterize $X_o$~\cite{halva2020hidden,hyvarinen2016unsupervised}.
Using d-separation, we can observe in Fig.~\ref{fig scm} that $X_o$ is d-separated by $D$ given $O$. Denoting this observation as $X_o \Perp D \vert O$, in the following, we show that it can introduce two effective strategies that can serve as complements for the classical marginal independent-based framework.

\textbf{First}, given the conditional independent observation $X_o \Perp D \vert O$, we present the following proposition:
\begin{prop}
\label{prop}
$X_o$ and $X_d$ from the same object should not have any association, indicating that when designing the network, the branch used for extracting the causal feature should not depend on the non-causal one.   
\end{prop}
\begin{proof}
\label{proof}
Given the optimal $F_o^{\ast}$, if the conditional independent constraint $X_o \Perp D \vert O$ is satisfied, then for $\forall$ domains $m, n$ and $\forall$ object $o$, we have $F_o^{\ast}(x^m) = F_o^{\ast}(x^n)$, where $x^m$ and $x^n$ correspond to the same object $o$. 
Now consider a simple counter-example. Given inputs $x'^m$, $x'^n$ from the same object $o'$, due to domain-dependent variation, the domain-specific (non-causal) features for these two inputs will be different: $F_d(x'^m) \neq F_d(x'^n)$. Since $F_o$ is a deterministic function regarding $x$ that is not independent of $X_d$, then for the same object that has different $X_d$, $F_o(x)$ will be different. In this case: $F_o(x'^m) \neq F_o(x'^n)$, violating the conditional independence constraint. Therefore, the optimal $F_o^{\ast}$ should not depend on the domain-specific feature. 
\end{proof}

\begin{figure}
\centering
	\begin{minipage}[b]{0.95\linewidth}
		\centering
		\centerline{
			\includegraphics[width =\linewidth]{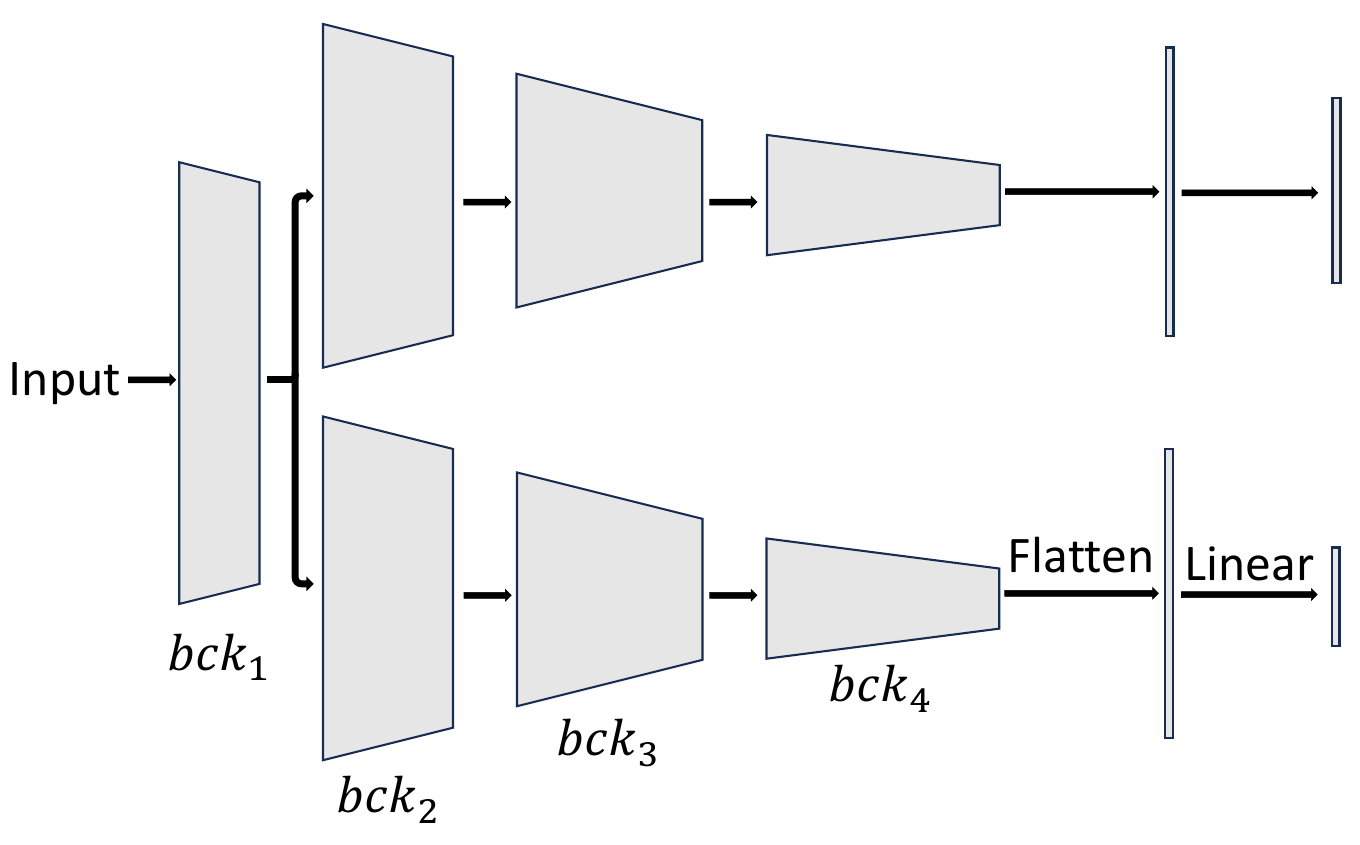}}
	\end{minipage}
	\caption{{An example diagram of the early branching structure upon the ResNet backbone \cite{he2016deep}. Here `$bck$' indicates the four layers in the ResNet backbone. Similar to that in Fig.~\ref{fig pipeline}, the top and bottom branches are for the semantic and domain estimation tasks, respectively.}}
	\vspace{-0.4 cm}
	\label{fig diag}
\end{figure}

 Based on the proposition, we argue that the common practice~\cite{ganin2016domain,albuquerque2019generalizing,atzmon2020causal,chen2021style} of using a shared base feature extractor with two lightweight predicting heads for the dual-branch structure is detrimental to the performance, as such design will result in $F_o$ to rely heavily on $X_o$. Hence, we suggest an early-branching structure where the two branches share the first few blocks and diverge thereafter to implement the dual-branch structure. Experimental support for this argument can be found in Sec.~\ref{sec early}. Note the sharing of early blocks for the two branches does not violate the constraints, because deep networks often learn low-level features in their early blocks, which are identical in $F_o$ and $F_d$ when extracting the high-level causal and non-causal features, and thus can be shared between them. {An example diagram of the early branching structure upon the ResNet backbone \cite{he2016deep} is shown in Fig.~\ref{fig diag}.} This design can also enable our design for the augmentation part, please refer to Sec.~\ref{sec aug} for details.

\textbf{Second}, the observation $X_o \Perp D \vert O$ implies that $X_o$ remains unchanged across different domains for the same object. However, implementing this concept can be challenging because the object information is often unavailable in existing datasets. Rather than heuristically pairing samples based on their feature distance~\cite{atzmon2020causal}, we consider each image as an object, and we suggest using augmentation to create samples that correspond to the same object but with different domains. Then, the consistency constraint can be fulfilled by requiring causal features from the original sample and its corresponding augmented version to be equal.  

Formally, we use $\mathcal{A}$ to denote a domain augmentation operator. As depicted in Fig.~\ref{fig aug}, in our design, we wish $\mathcal{A}$ could modify images in a way that changes the non-causal domain information of the given sample while preserving the causal semantic information. Then the conditional consistency is fulfilled by minimizing the following,
\begin{equation}
\label{eq contrast}
\begin{aligned}
 \mathcal{L}_{cons} = &\Vert F_o(x) - \text{MLP}(F_o(\mathcal{A}(x))) \Vert_1 \\ &+ \Vert \text{MLP}(F_o(x)) - F_o(\mathcal{A}(x)) \Vert_1,
\end{aligned}
\end{equation}
where $\Vert \cdot \Vert_1$ is the $\ell_1$ norm. We use $\|\text{MLP}(\mathbf{a}) - \mathbf{b}\|$ to measure the compatibility between $\mathbf{a}$ and $\mathbf{b}$, which can prevent representation collapse \cite{grill2020bootstrap,kim2021selfreg}. In our design, $\mathcal{A}$ is operated at the feature level and applied to the features extracted from $\theta_b$ (\ie $\mathcal{A}(x) \triangleq \mathcal{A}(\theta_b(x))$) so that both the two branches can be benefited. We detailed the operator $\mathcal{A}$ in the following section.

\begin{figure}
\centering
	\begin{minipage}[b]{0.78\linewidth}
		\centering
		\centerline{
			\includegraphics[width =\linewidth]{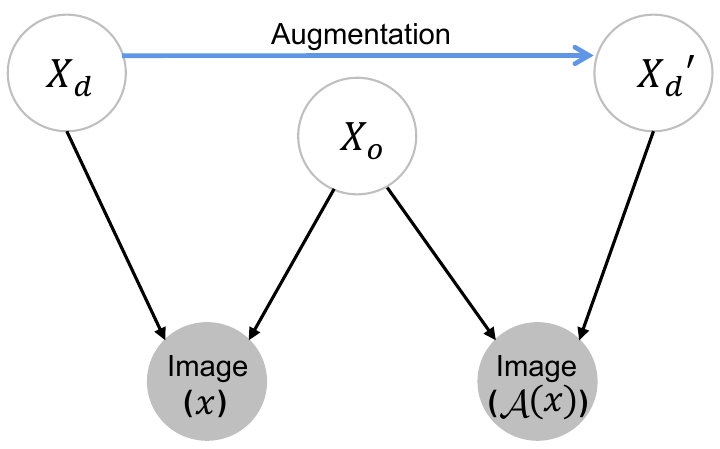}}
	\end{minipage}
	\caption{Overview of our augmentation strategy. We wish the augmentation operator $\mathcal{A}$ to change the non-causal domain information of a sample while keeping the causal feature unchanged.}
	\vspace{-0.4 cm}
	\label{fig aug}
\end{figure}

\subsection{A New Domain Augmentation Strategy}
\label{sec aug}
To better fulfill the conditional consistency regularization, we design a random domain sampling (RDS) strategy to augment the domain of a sample while preserving its semantic information. The skill in \cite{zhou2021domain} is adopted for implementation, which shows that new domains can be obtained by mixing the styles of two different samples. In \cite{zhou2021domain}, the styles are the mean $\mu \in \mathbb{R}^{b\times c}$ and standard derivation $\sigma^2 \in \mathbb{R}^{b\times c}$ of activations across all spatial locations and all samples in a batch, given a set of $b\times c\times h\times w$ sized feature maps, where $b, c, h$, and $w$ are batch, channel, height, and width of the feature map.   

However, modifying $\mu$ and $\sigma^2$ by directly mixing with other image styles may create more homogeneous features when confronting similar domain types within a batch, resulting in the same $X_d$. RDS solves this problem by perturbing the styles with a controllable strength. Specifically, we first calculate the mean and variance of activations from each image, denoted as $\mu_i$ and $\sigma^2_i$. Then, we build a probabilistic model for $\mu_i$ and $\sigma^2_i$, which is done by assuming they follow the Gaussian distribution and then estimating the model parameters within a batch. Finally, we can sample new $\mu$ and $\sigma^2$ from the Gaussian distribution as new styles. Take $\mu$ as an example. The above procedure can be written as:
\begin{equation}
\begin{small}
\label{eq style}
\mu \sim \mathcal{N}(\mathop{\mathbb{E}}\nolimits_{\mu}, \Sigma_{\mu}),~~\text{s.t.}
\begin{cases}
\begin{aligned}
&\mathop{\mathbb{E}}\nolimits_{\mu} = \frac{1}{b} \sum_{i=1}^{b}\mu_i,\\
&\Sigma_{\mu}=\frac{1}{b} \sum_i^b (\mu_i - \mathop{\mathbb{E}}\nolimits_{\mu})(\mu_i - \mathop{\mathbb{E}}\nolimits_{\mu})^{\text{T}}.\\
\end{aligned}
\end{cases}
\end{small}
\end{equation} 
To obtain inhomogeneous styles, we only accept sampled $\mu$ whose density is less than $\epsilon$. The accepted samples, denoted as $\hat{\mu}$ should satisfy,
\begin{equation}
\begin{aligned}
\label{eq sample}
\frac{1}{(2\pi)^{\frac{b\times c}{2}} \vert \Sigma_{\mu} \vert^{\frac{1}{2}}} \exp \left( -\frac{1}{2} (\hat{\mu} - \mathop{\mathbb{E}}\nolimits_{\mu})^{\text{T}} \Sigma_{\mu}^{-1} (\hat{\mu} - \mathop{\mathbb{E}}\nolimits_{\mu}) \right) < \epsilon ,
\end{aligned}
\end{equation}
where $\epsilon$ is sufficiently small (\ie set to be 0.0001 in our implementation) so that the sampled $\hat{\mu}$ can be distinct from the original $\mu$. An illustration of RDS is shown in Fig.~\ref{fig samp}. We apply the same sampling process to obtain $\hat{\sigma}$.

After sampling new styles, we follow the protocol in AdaIN \cite{huang2017arbitrary} by replacing the styles from the original representation with the sampled ones, and the final formulation for our domain augmentation can be written as,
\begin{equation}
\label{eq rds}
\mathcal{A}(x) = \hat{\sigma} \frac{\theta_b(x) - \mu}{\sigma} + \hat{\mu}.
\end{equation}
Different from the previous work \cite{li2022uncertainty,zhou2021domain} that specifically adds normal perturbations to their styles or borrows that from others, RDS manages to control how the sampled styles deviated from the existing ones thanks to its probabilistic modeling of $\mu_i$ and $\sigma^2_i$, thus preventing the augmentation strategy from creating more homogeneous domains. We provide more detailed comparisons between our RDS and previous ideas that are based on AdaIN (\ie MixStyle \cite{zhou2020learning} and DSU \cite{li2022uncertainty}) in Sec.~\ref{sec aug} and~\ref{sec moreaug}. 

\begin{figure}
\centering
	\begin{minipage}[b]{0.8\linewidth}
		\centering
		\centerline{
			\includegraphics[width =\linewidth]{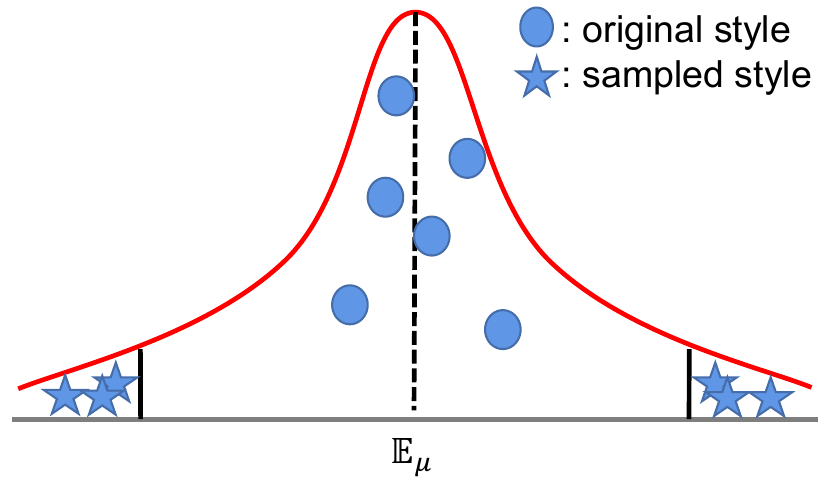}}
	\end{minipage}
	\caption{A simple demonstration of the proposed random domain sampling (RDS) strategy.}
	\label{fig samp}
 \vspace{-0.3 cm}
\end{figure}

\subsection{Learning Objective}
Without the loss of generality, we also impose independent constraints on the augmented features. Combining the standard classification loss $\mathcal{L}_{cls}$:
\begin{equation}
\mathcal{L}_{cls} = \sum_{\{x_1\} \in \{x, \mathcal{A}(x)\}}\text{CE}(H_o(F_o(x_1)), y) + \text{CE}(H_d(F_d(x)), d),
\end{equation}
where CE denoting the cross-entropy loss. Our training objective is to find $F_o^{\ast}$ and $F_d^{\ast}$ that minimizes $\mathcal{L}_{all}$,
\begin{equation}
\begin{aligned}
\label{eq allloss}
\mathcal{L}_{all} = \mathcal{L}_{cls} + \alpha \mathcal{L}_{indp} + \beta \mathcal{L}_{cons},
\end{aligned}
\end{equation}
where $\mathcal{L}_{indp} = \sum_{\{x_1, x_2\} \in \{x, \mathcal{A}(x)\}}\text{HSIC}(F_o(x_1), F_d(x_2))$ derives from the basic marginal independence constraint; $\alpha$ and $\beta$ are positive weights. Note the augmented features do not contain domain labels. In order to enable $F_d^{\ast}$ to be informed of the augmented features, we further expand $\mathcal{L}_{cons}$ by adding an inverse consistency constraint: $ - ([\Vert F_d(x) - \text{MLP}(F_d(\mathcal{A}(x)))\Vert_1 - n]_{-} + [\Vert \text{MLP}(F_d(x)) - F_d(\mathcal{A}(x)) \Vert_1 - n]_{-})$, where $[\cdot]_- = \min(\cdot,0)$, and the margin $n$ is set to be the maximum value of the $\ell_1$ distances computed from the current batch~\cite{yang2021adversarial}.   

\begin{table*}[t]
\centering
\caption{Evaluations on the DomainBed benchmark \cite{gulrajani2020search} using the ResNet18 backbone. All methods are examined for $3 \times 20$ times in each unseen domain. Here Top5 accumulates the number of datasets where a method achieves the top 5 performances; Score here accumulates the numbers of the dataset where a specific art performs better than ERM; Best and second best results are presented in \textbf{bold} and \underline{undelined} types. Our method performs favorably against existing arts regarding all 3 criteria.}
\scalebox{1}{
\begin{tabular}{lC{1.8cm}C{1.8cm}C{1.8cm}C{1.8cm}C{1.8cm}C{0.7cm}|C{0.7cm} C{0.7cm}}
\toprule 
& PACS & VLCS & OfficeHome & TerraInc & DomainNet & Avg. & Top5 &Score\\
\hline \hline
MMD \cite{li2018domain} &81.3  $\pm$  0.8 &74.9  $\pm$  0.5 &59.9  $\pm$  0.4 &42.0  $\pm$  1.0 &7.9  $\pm$  6.2 &53.2 &1 &2\\
RSC \cite{huang2020self} &80.5  $\pm$  0.2 &75.4  $\pm$  0.3 &58.4  $\pm$  0.6 &39.4  $\pm$  1.3 &27.9  $\pm$  2.0 &56.3 &0 &2\\
IRM \cite{arjovsky2019invariant} &80.9  $\pm$  0.5 & 75.1  $\pm$  0.1 &58.0  $\pm$  0.1 &38.4  $\pm$  0.9 &30.4  $\pm$  1.0 &56.6 &0 &1\\
ARM \cite{zhang2020adaptive} &80.6  $\pm$  0.5 &75.9  $\pm$  0.3 &59.6  $\pm$  0.3 &37.4  $\pm$  1.9 &29.9  $\pm$  0.1 &56.7 &0 &2 \\
DANN \cite{ganin2016domain} &79.2  $\pm$  0.3 &76.3  $\pm$  0.2 &59.5  $\pm$  0.5 &37.9  $\pm$  0.9 &31.5  $\pm$  0.1 &56.9 &1 &1 \\
GroupGRO \cite{sagawa2019distributionally} &80.7  $\pm$  0.4 &75.4  $\pm$  1.0 &60.6  $\pm$  0.3 &41.5  $\pm$  2.0 &27.5  $\pm$  0.1 &57.1 &0 &2 \\
CDANN \cite{li2018deep} &80.3  $\pm$  0.5 &76.0  $\pm$  0.5 &59.3  $\pm$  0.4 &38.6  $\pm$  2.3 &31.8  $\pm$  0.2 &57.2 &0 &2 \\
VREx \cite {krueger2021out} &80.2  $\pm$  0.5 &75.3  $\pm$  0.6 &59.5  $\pm$  0.1 &\textbf{43.2  $\pm$  0.3} &28.1  $\pm$  1.0 &57.3 &1 &2 \\
CAD \cite{ruan2021optimal} &81.9  $\pm$  0.3 &75.2  $\pm$  0.6 &60.5  $\pm$  0.3 &40.5  $\pm$  0.4 &31.0  $\pm$  0.8 &57.8 &1 &2 \\
CondCAD \cite{ruan2021optimal} &80.8  $\pm$  0.5 &76.1  $\pm$  0.3 &61.0  $\pm$  0.4 &39.7  $\pm$  0.4 &31.9  $\pm$  0.7 &57.9 &0 &4\\
MTL \cite{blanchard2021domain} &80.1  $\pm$  0.8 &75.2  $\pm$  0.3 &59.9  $\pm$  0.5 &40.4  $\pm$  1.0 &35.0  $\pm$  0.0 &58.1 &0 &2\\
ERM \cite{vapnik1999nature} &79.8  $\pm$  0.4 &75.8  $\pm$  0.2 &60.6  $\pm$  0.2 &38.8  $\pm$  1.0 &35.3  $\pm$  0.1 &58.1 &0 &-\\
MixStyle \cite{zhou2021domain} &\textbf{82.6  $\pm$  0.4} &75.2  $\pm$  0.7 &59.6  $\pm$  0.8 &40.9  $\pm$  1.1 &33.9  $\pm$  0.1 &58.4 &1 &2\\
MLDG \cite{li2018learning} &81.3  $\pm$  0.2 &75.2  $\pm$  0.3 &60.9  $\pm$  0.2 &40.1  $\pm$  0.9 &35.4  $\pm$  0.0 &58.6 &1 &4 \\
Mixup \cite{yan2020improve} &79.2  $\pm$  0.9 &76.2  $\pm$  0.3 &61.7  $\pm$  0.5 &42.1  $\pm$  0.7 &34.0  $\pm$  0.0 &58.6 &2 &3 \\
Fishr \cite{rame2021ishr} &81.3  $\pm$  0.3 &76.2  $\pm$  0.3 &60.9  $\pm$  0.3 &\underline{42.6  $\pm$  1.0} &34.2  $\pm$  0.3 &59.0 &2 &4 \\
SagNet \cite{nam2021reducing} &81.7  $\pm$  0.6 &75.4  $\pm$  0.8 &\underline{62.5  $\pm$  0.3} &40.6  $\pm$  1.5 &35.3  $\pm$  0.1 &59.1 &1 &3 \\
SelfReg \cite{kim2021selfreg} &81.8  $\pm$  0.3 &76.4  $\pm$  0.7 &62.4  $\pm$  0.1 &41.3  $\pm$  0.3 &34.7  $\pm$  0.2 &59.3 &2 &4 \\
Fish \cite{shi2021gradient} &82.0  $\pm$  0.3 &\textbf{76.9  $\pm$  0.2} &62.0  $\pm$  0.6 &40.2  $\pm$  0.6 &35.5  $\pm$  0.0 &59.3 &3 &5\\
CORAL \cite{sun2016deep} &81.7  $\pm$  0.0 &75.5  $\pm$  0.4 &62.4  $\pm$  0.4 &41.4  $\pm$  1.8 &\underline{36.1  $\pm$  0.2} &59.4 &2 &4 \\
SD \cite{pezeshki2021gradient} &81.9  $\pm$  0.3 &75.5  $\pm$  0.4 &\textbf{62.9  $\pm$  0.2} &42.0  $\pm$  1.0 &\textbf{36.3  $\pm$  0.2} &\underline{59.7} &4 &4\\
Ours &\underline{82.4  $\pm$  0.4} &\underline{76.5  $\pm$  0.4} &62.2  $\pm$  0.1 &\textbf{43.2  $\pm$  1.3} &34.9  $\pm$  0.1 &\textbf{59.8} &4 &4 \\
\bottomrule
\end{tabular}}
\label{tab results}
\vspace{-0.2 cm}
\end{table*}

\subsection{Indentifiability}
A general causal model would require identifiable for latent variables $X_o$ from observed image-label pairs $(x,y)$, which is known to be an ill-posed problem given only limited observed data. In our case, the causal feature $X_o$ is assumed invariant for the same object in different domains, and vice versa for the non-causal feature $X_d$, as it corresponds to the domain-specific information. Hence, the change of domain can be viewed as an unknown intervention that alters the domain-specific information of an object, but the semantic information, which is required for correct classification, remains unchanged. Our augmentation strategy is consistent with this setting, which further aligns with the assumption in K{\"u}gelgen \etal \cite{von2021self}. 
Thus according to Theorem 4.4 in \cite{von2021self}, the identifiability of $X_o$ can be well established if sufficient augmentations are applied and $F_o$ constrained with a consistency regularization (\ie $\mathcal{L}_{cons}$ in Eq.~\eqref{eq contrast}). Note the self-supervised entropy maximization task in \cite{von2021self} is replaced by the adopted supervised objective $\mathcal{L}_{cls}$, which does not diminish the identifiability, because the label information $\{y\}$ can also help purify $X_o$ as $X_o \not\Perp y$ \cite{mahajan2021domain}.

\section{Experiments}

\subsection{Evaluation on DomainBed}
\subsubsection{Datasets and Details}
\label{sec details}
\noindent\textbf{Datasets.} To evaluate the effectiveness of our method, we conduct extensive experiments on five DG benchmarks, namely PACS \cite{li2017deeper}, VLCS \cite{fang2013unbiased}, OfficeHome \cite{venkateswara2017deep}, TerraInc \cite{beery2018recognition}, and DomainNet~\cite{peng2019moment}. Specifically, PACS consists of 9,991 images which can be divided into 7 classes. This dataset is the most commonly used DG benchmark due to its large distributional shift across 4
domains including art painting, cartoon, photo, and sketch;
VLCS collects a total of 10,729 images from 4 different datasets (\ie PASCAL VOC 2007 \cite{everingham2010pascal}, LabelMe \cite{russell2008labelme}, Caltech \cite{fei2004learning}, and Sun \cite{xiao2010sun}) which can be categorized into 5 classes. Images from different datasets are taken under different views and considered different domains;
OfficeHome is an object recognition dataset that consists of 15,588 images from 65 classes in office and home environments, and these images can be divided into 4 domains including artistic, clipart, product, and real world;
TerraInc contains 24,788 images of wild animals from 10 categories. Those images are taken from 4 different locations (\ie L100, L38, L43, L46), and the locations are regarded as the varying domains in the setting;
DomainNet contains 586,575 images from a total of 345 classes whose domains can be depicted in 6 styles, including clipart, infograph, painting, quickdraw, real, and sketch.

\begin{table*}
\centering
\caption{{Average accuracies on the DomainBed \cite{gulrajani2020search} benchmark using the ResNet50 \cite{he2016deep} backbones. The top part uses the default settings in \cite{gulrajani2020search} with results borrowed from \cite{chen2023domain} (except for CBDG, which is from the original paper), and the bottom part uses settings from SWAD~\cite{cha2021swad} with results borrowed from \cite{cha2021swad}.}} 
\scalebox{1}{
\begin{tabular}{lC{1.7cm}C{1.7cm}C{1.7cm}C{1.7cm}C{1.7cm}C{0.9cm}}
\toprule 
& PACS & VLCS & OfficeHome & TerraInc & DomainNet & Avg.\\
\hline \hline
$\text{ERM}$ \cite{vapnik1999nature} &83.1 $\pm$ 0.9 &77.7 $\pm$ 0.8 &65.8 $\pm$ 0.3 &46.5 $\pm$ 0.9 &40.8 $\pm$ 0.2 &62.8\\
$\text{Fish}$ \cite{shi2021gradient} &84.0 $\pm$ 0.3 &78.6 $\pm$ 0.1 &67.9 $\pm$ 0.5 &46.6 $\pm$ 0.4 &40.6 $\pm$ 0.2 &63.5\\
$\text{CORAL}$ \cite{sun2016deep} &85.0 $\pm$ 0.4 &77.9 $\pm$ 0.2 &68.8 $\pm$ 0.3 &46.1 $\pm$ 1.2 &41.4 $\pm$ 0.0 &63.9\\
$\text{SD}$ \cite{pezeshki2021gradient}  &84.4 $\pm$ 0.2 &77.6 $\pm$ 0.4 &68.9 $\pm$ 0.2 &46.4 $\pm$ 2.0 &42.0 $\pm$ 0.2 &63.9\\
$\text{CBDG}$ \cite{wang2022causal}  &86.1 $\pm$ 0.4 &76.1 $\pm$ 0.3 &67.1 $\pm$ 0.4 &48.8 $\pm$ 1.7 &42.6 $\pm$ 1.0 &64.1\\
$\text{RIDG}$ \cite{chen2023domain}  &84.7 $\pm$ 0.2 &77.8 $\pm$ 0.4 &68.6 $\pm$ 0.2 &47.8 $\pm$ 1.1 &41.9 $\pm$ 0.3 &64.2\\
$\text{Ours}$ &85.6 $\pm$ 0.3 & 78.4 $\pm$ 0.1 &68.5 $\pm$ 0.4 &47.3 $\pm$ 1.6 &40.5 $\pm$ 0.1 &64.1\\
\hline
Settings in SWAD~\cite{cha2021swad}\\
\hline
ERM &84.2 $\pm$ 0.1 &77.3 $\pm$ 0.1 &67.6 $\pm$ 0.2 &47.8 $\pm$ 0.6 &44.0 $\pm$ 0.1 &64.2\\
CORAL &86.2 $\pm$ 0.3 &78.8 $\pm$ 0.6 &68.7 $\pm$ 0.3 &47.6 $\pm$ 1.0 &41.5 $\pm$ 0.1 &64.5\\
Ours &86.5 $\pm$ 0.4 &78.4 $\pm$ 0.2 &69.7 $\pm$ 0.5 &48.2 $\pm$ 0.8 &43.3 $\pm$ 0.2 &65.2\\
ERM+SWAD &88.1 $\pm$ 0.1 &79.1 $\pm$ 0.1 &70.6 $\pm$ 0.2 &50.0 $\pm$ 0.3 &46.5 $\pm$ 0.1 &66.9\\
CORAL+SWAD &88.3 $\pm$ 0.1 &78.9 $\pm$ 0.1 &71.3 $\pm$ 0.1 &51.0 $\pm$ 0.1 &46.8 $\pm$ 0.0 &67.3\\
Ours+SWAD &89.0 $\pm$ 0.4 &79.5 $\pm$ 0.1 &71.1 $\pm$ 0.2 &52.8 $\pm$ 0.5 &46.3 $\pm$ 0.0 &67.7\\
\bottomrule
\end{tabular}}
\label{tab res50}
\vspace{-0.5 cm}
\end{table*}

\noindent\textbf{Implementation and evaluation details.}
We use both the ResNet18 \cite{he2016deep} (ImageNet pretrained) and Resnet50 for evaluations. The hyper-parameters in our model include $\alpha$ and $\beta$ in Eq.~(\ref{eq allloss}). We use a dynamic range of $\alpha \in \left[0.1, 1 \right]$ and $\beta \in \left[0.03, 0.3\right]$ for our experiments. Settings for all the compared methods are set according to \cite{gulrajani2020search}. The leave-one-out strategy is used for the training and evaluation process.
During training, we split the examples from training domains to 8:2 (train:val). The learning rates, augmentation strategies, random seeds, and batch sizes are all dynamically set in reasonable ranges, and the iteration step is fixed at 5,000 for each training.
During the evaluation, we select the model that performs the best in the validation samples and test it on the whole held-out domain, which is also referred to as the "training-domain validate set" model selection method in \cite{gulrajani2020search}. We evaluate all the methods $3 \times 20$ times in each target domain: a total of 3 trials and each with 20 randomly selected hyper-parameters unless otherwise mentioned. For each domain in different datasets, the final performance is the average accuracy computed from the 3 trials that are with the best performing hyper-parameter setting, and the performance in this dataset is the average of the corresponding domains. To ensure fair comparisons, all the methods are evaluated in the same device (8 Nvidia Tesla v100 GPUs with 32G memory). 

\begin{table*}[t]
    \centering
    \caption{{Evluations on four challenging datasets from the Wilds benchmark~\cite{koh2021wilds}. Results are directly cited from the public leaderboard and the original paper. Metrics of means and standard deviations are reported across different trials (according to the default settings in \cite{koh2021wilds}). The best results are colored as \red{red}.}}
    \scalebox{1}{
    \begin{tabular}{lC{1.8cm}C{1.8cm}C{1.8cm}C{1.8cm}C{1.8cm}C{1.8cm}}
\toprule 
& \multicolumn{2}{c}{iWildCam}\hspace{0pt} & Camelyon17\hspace{0pt} & RxRX1 & \multicolumn{2}{c}{FMoW} \\
  \cmidrule(lr){2-3} \cmidrule(lr){4-4} \cmidrule(lr){5-5} \cmidrule(lr){6-7}
    Method &Avg. acc. &Macro F1 &Avg. acc. & Avg. acc. &Worst acc.  & Avg. acc. \\
    \hline \hline
    ERM \cite{vapnik1999nature} & 71.6 $\pm$ 2.5 &31.0 $\pm$ 1.3 &70.3 $\pm$ 6.4 &29.9 $\pm$ 0.4 &32.3 $\pm$ 1.25 &53.0 $\pm$ 0.55\\
    CORAL~\cite{sun2016deep} &\red{73.3 $\pm$ 4.3} &\red{32.8 $\pm$ 0.1} &59.5 $\pm$ 7.7 &28.4 $\pm$ 0.3 &31.7 $\pm$ 1.24 &50.5 $\pm$ 0.36\\
    GroupGRO \cite{sagawa2019distributionally} & 72.7 $\pm$ 2.1 &23.9 $\pm$ 2.0 &68.4 $\pm$ 7.3 &23.0 $\pm$ 0.3 &30.8 $\pm$ 0.81 &52.1 $\pm$ 0.50\\
    IRM~\cite{arjovsky2019invariant} & 59.8 $\pm$ 3.7 &15.1 $\pm$ 4.9 &64.2 $\pm$ 8.1 &8.2 $\pm$ 1.1 &30.0 $\pm$ 1.37 &50.8 $\pm$ 0.13\\
    ARM-BN \cite{zhang2020adaptive} & 70.3 $\pm$ 2.4 &23.7 $\pm$ 2.7 &87.2 $\pm$ 0.9 &\red{31.2 $\pm$ 0.1} &24.6 $\pm$ 0.04 &42.0 $\pm$ 0.21\\
    TTBNA~\cite{schneider2020improving} & 46.6 $\pm$ 0.9 &13.8 $\pm$ 0.6 &- &20.1 $\pm$ 0.2 &30.0 $\pm$ 0.23 &51.5 $\pm$ 0.25\\
    Fish~\cite{shi2021gradient} & 64.7 $\pm$ 2.6 &22.0 $\pm$ 1.8 &74.7 $\pm$ 7.1 &- &34.6 $\pm$ 0.18 &51.8 $\pm$ 0.32\\
    CGD~\cite{piratla2021focus} &- &- &69.4 $\pm$ 7.9 &- &32.0 $\pm$ 2.26 &50.6 $\pm$ 1.39\\
    LISA~\cite{yao2022improving} & - &- &77.1 $\pm$ 6.9 &31.9 $\pm$ 1.0 &35.5 $\pm$ 0.81 &52.8 $\pm$ 1.15\\
    RIDG~\cite{chen2023domain} & 70.9 $\pm$ 2.8 &30.7 $\pm$ 1.0 &90.6 $\pm$ 2.9 &30.0 $\pm$ 0.3 &\red{36.1 $\pm$ 1.48} &\red{55.9 $\pm$ 0.25}\\
    Ours & 71.8 $\pm$ 2.0 &31.5 $\pm$ 1.4 &\red{91.0 $\pm$ 4.4} &30.4 $\pm$ 0.1 &{35.0 $\pm$ 0.73} &{53.9 $\pm$ 0.40}\\
    \bottomrule
    \end{tabular}}
    \label{tab wilds}
\end{table*}

\subsubsection{Experimental Results}
\label{sec domainbedresults}
\textbf{Results with ResNet18.}
We report the average accuracy (\ie Avg.), leading performance (\ie Top5), and improvements regarding the baseline (\ie Score) for all compared methods.
Results are shown in Table~\ref{tab results}. We observe that the naive ERM method obtains comparable results against existing arts. In fact, as a strong baseline, ERM is ranked 11th place in the term of average accuracy, leading half of the compared model, and less than half of the state-of-the-art methods can outperform ERM in most datasets (\ie Score $\geq$ 3). These observations are in keeping with the findings in~\cite{gulrajani2020search}. Meanwhile, we note that when using sophisticated designs to learn domain-invariant features on the basis of ERM, the average performances seem to decrease in most cases for IRM \cite{arjovsky2019invariant}, MMD \cite{li2018domain}, DANN \cite{ganin2016domain}, CDANN \cite{li2018deep}, and RSC \cite{huang2020self}, indicating the ineffectiveness of current strategies that are specially designed for maintaining invariance. 

In contrast, our causal-inspired method outperforms ERM in most datasets. Specifically, it obtains the top 5 performances in 4 out the 5 datasets evaluated and is ranked 1st place in the term of average accuracy of all 5 benchmarks. When compared with alternatives that either use augmentations (\ie Mixup \cite{yan2020improve}, SagNet \cite{nam2021reducing}, SelfReg \cite{kim2021selfreg}, MixStyle \cite{zhou2021domain}, and CAD \cite{ruan2021optimal}) or gradient-level constraints (\ie, Fish \cite{shi2021gradient}, SD \cite{pezeshki2021gradient}, and Fishr \cite{rame2021ishr}), our method consistently shows comparable or even better performance in most datasets. These results validate the effectiveness of our complete framework. 
{There is only one dataset where our method fails in competing with other leading algorithms. We note that in the DomainNet dataset, similar to those methods involving data augmentations, our method cannot outperform the baseline ERM. This may indicate that this dataset is insensitive to existing augmentation skills, and a more effective customized alternative may be required. We leave this open problem to our future research.}
We present results of average accuracy in each domain from different datasets in the supplementary material. Please refer to it for details.

\noindent{\textbf{Results with ResNet50.}
To comprehensively evaluate the proposed method, we also test it with a larger ResNet50 backbone. In this setting,  our method is compared with the baseline and some leading methods from Table~\ref{tab results}. Some recent state-of-the-art methods (\ie RIDG~\cite{chen2023domain} and CBDG~\cite{wang2022causal}) are introduced for comparison as well. As larger networks require more resources, each trail will correspond to 5 hyper-parameter searches in this setting.     
As can be observed in Table~\ref{tab res50}, our method obtains competitive results against recent methods~\cite{chen2023domain,wang2022causal}. Meanwhile, our method also improves the baseline in most adopted datasets. This improvement indicates its effectiveness in generalization when the backbone size increases.}

{We also reevaluate our method following the settings in SWAD~\cite{cha2021swad} (i.e. implemented with smaller hyper-parameter searching space) and list the results in Table~\ref{tab res50}. Similar to other methods, our method is not as effective as SWAD due to a lack of seeking flat minima. On the other hand, since SWAD can be integrated with other methods to improve the performances, we further evaluate our method by combining it with SWAD (i.e. Ours+SWAD). The results in Table~\ref{tab res50} indicate that when combined with SWAD, our method obtains competitive performance against others with the same strategy.
}

\subsection{Evaluation on Wilds}
\label{sec wilds}
\subsubsection{Datasets and Details}
{\textbf{Datasets.} Wilds contain multiple datasets displaying real-world distribution shifts. Following RIDG~\cite{chen2023domain}, we adopt 4 challenging datasets for evaluation. The dataset details are listed as follows: \textbf{(1) iWildCam~\cite{beery2021iwildcam}} contains 203,029 animal images from a total of 182 species, and there are 324 domains in this dataset. These domains are trapped by the camera in different locations; \textbf{(2) Camelyon17~\cite{bandi2018detection}} consists of 45,000 images used for the binary tumor classification task. The 5 hospitals that collect these images are regarded as domains; \textbf{(3) RxRx1~\cite{taylor2019rxrx1}} contains 125,514 high-resolution fluorescence microscopy images of human cells under 1,108 genetic perturbations, which comprise a total of 51 experimental batches (i.e. domains); \textbf{(4) FMoW~\cite{christie2018functional}} is a satellite dataset contains 118,886 samples used
for land classification across 62 regions and 80 years, where the regions and years are regarded as category and domain, respectively.}

\noindent{\textbf{Implementation and evaluation details.} Following the default practice in~\cite{koh2021wilds}, we use different settings for different datasets during implementation. Specifically, in iWildCam and RxRx1 datasets, the imagenet pretrained ResNet50 backbone is utilized for a total of 3 trials. We report average accuracy (i.e. avg. acc.) for the two datasets and Macro F1 for iWildCam. For Camelyon17 and FMoW, the imagenet pretrained DenseNet121~\cite{huang2017densely} is adopted for evaluation. In Camelyon17, we conduct 10 trials to compute the average accuracy. In FMoW, we report both the worst-case and average accuracies, which are computed over 3 trials.}  

\subsubsection{Experimental Results}
{Table~\ref{tab wilds} shows our experimental results. We use the statistics listed in the public leaderboard\footnote{https://wilds.stanford.edu/leaderboard/}  (i.e. CORAL~\cite{sun2016deep},
GroupGRO \cite{sagawa2019distributionally},
IRM~\cite{arjovsky2019invariant},
ARM-BN \cite{zhang2020adaptive},
TTBNA~\cite{schneider2020improving},
Fish~\cite{shi2021gradient},
LISA~\cite{yao2022improving}) and the original paper (i.e. RIDG~\cite{chen2023domain}) for the method comparison.
Similar to the observation in the Domainbed benchmark, we note that most current methods perform on par with the baseline ERM method, indicating that when confronting difficult generalization tasks, these empirical designs may not be as effective as the simplest end-to-end training strategy. 
Differently, our method is guided by the theoretical guidance from the causal relations and can help to improve the baseline in almost all datasets. Moreover, compared to the previous leading art RIDG~\cite{chen2023domain}, our algorithm can obtain better performance in 3 out of the 4 datasets. These results validate its effectiveness in the challenging Wilds benchmark.}

\begin{table}[t]
\centering
\caption{Evaluations of ERM and the proposed basic framework with different settings of $\theta_b$ in the unseen domain from \cite{li2017deeper} in ResNet backbones. Here $bck_{1,2,3,4}$ is the standard four blocks in the original implementation with $bck_4$ close to the classifier. The common practice of using a shared base feature extractor with two lightweight prediction heads is detrimental to the performance, and the proposed early-branching structure leads to better results.}
\scalebox{0.85}{
\begin{tabular}{c|ccccc}
\toprule 
& Art & Cartoon & Photo & Sketch &Avg.\\
\hline \hline
\multicolumn{6}{c}{~~~~~~~~~~~~~~~~~~ResNet18 backbone}\\
\hline
ERM &78.0 $\pm$ 1.3 &73.4 $\pm$ 0.8 &94.1 $\pm$ 0.4 &73.6 $\pm$ 2.2 &79.8 $\pm$ 0.4\\
\hline
$\theta_b=\textbf{I}$ &79.1 $\pm$ 0.7 &74.1 $\pm$ 1.1 &94.6 $\pm$ 0.4 &74.8 $\pm$ 0.6 &80.7 $\pm$ 0.5\\
$\theta_b= bck_{1}$  &79.3 $\pm$ 0.9 &74.1 $\pm$ 2.0 &94.8 $\pm$ 0.5 &75.8 $\pm$ 1.6 &81.0 $\pm$ 0.6\\
$\theta_b= bck_{2}$ &79.1 $\pm$ 1.2 &73.2 $\pm$ 1.8 &94.8 $\pm$ 0.6 &74.9 $\pm$ 1.3 &80.5 $\pm$ 0.8\\
$\theta_b= bck_{3}$ &78.3 $\pm$ 0.2 &72.8 $\pm$ 0.6 &94.9 $\pm$ 0.9 &74.1 $\pm$ 0.9 &80.0 $\pm$ 0.2\\
$\theta_b= bck_{4}$ &75.3 $\pm$ 0.7 &72.1 $\pm$ 0.2 &94.3 $\pm$ 0.2 &71.5 $\pm$ 0.6 &78.3 $\pm$ 0.3\\
\hline
\multicolumn{6}{c}{~~~~~~~~~~~~~~~~~~ResNet50 backbone}\\
\hline
ERM &85.1 $\pm$ 1.3 &78.0 $\pm$ 1.2 &97.3 $\pm$ 0.1 &72.2 $\pm$ 1.6 &83.1 $\pm$ 0.9\\
$\theta_b=\textbf{I}$ &85.8 $\pm$ 0.9 &77.6 $\pm$ 0.1 &97.2 $\pm$ 0.2 &77.7 $\pm$ 0.4 &84.5 $\pm$ 0.3\\
$\theta_b= bck_{1}$  &85.6 $\pm$ 1.0 &77.9 $\pm$ 1.7 &97.5 $\pm$ 0.4 &77.7 $\pm$ 0.1 &84.7 $\pm$ 0.5\\
$\theta_b= bck_{2}$ &84.2 $\pm$ 0.7 &77.7 $\pm$ 1.1 &96.3 $\pm$ 0.2 &73.5 $\pm$ 0.5 &82.9 $\pm$ 0.4\\
$\theta_b= bck_{3}$ &81.2 $\pm$ 1.2 &77.1 $\pm$ 0.6 &97.1 $\pm$ 0.1 &72.6 $\pm$ 0.7 &82.0 $\pm$ 0.5\\
$\theta_b= bck_{4}$ &85.4 $\pm$ 1.6 &78.0 $\pm$ 0.8 &97.1 $\pm$ 0.2 &66.4 $\pm$ 2.4 &81.7 $\pm$ 0.7\\
\bottomrule
\end{tabular}}
\label{tab base_results}
\end{table}

\begin{table*}[t]
\centering
\caption{Evaluations of the early-branching structure with a lightweight four layer network in the unseen domain from RMNIST \cite{ghifary2015domain}. Here $bck_{1,2,3,4}$ is the four convolution layers with $bck_4$ close to the classifier. The proposed framework achieves the optimal performances with an early-branching network structure (\ie $\theta_b=bck_1$ or $\theta_b=\textbf{I}$), similar as that observed in the ResNet backbones.}
\scalebox{1}{
\begin{tabular}{c|ccccccc}
\toprule 
& 0 & 15 & 30 & 45 &60 &75 &avg\\
\hline \hline
ERM &95.7 $\pm$ 0.2 &98.6 $\pm$ 0.1 &98.9 $\pm$ 0.1 &99.0 $\pm$ 0.1 &98.8 $\pm$ 0.0 &96.0 $\pm$ 0.1 &97.9 $\pm$ 0.1\\
\hline
$\theta_b=\textbf{I}$ &95.8 $\pm$ 0.0 &98.1 $\pm$ 0.1 &98.8 $\pm$ 0.0 &99.4 $\pm$ 0.1 &99.0 $\pm$ 0.0 &96.6 $\pm$ 0.2 &98.0 $\pm$ 0.1\\
$\theta_b= bck_{1}$   &96.2 $\pm$ 1.0 &98.4 $\pm$ 0.1 &99.1 $\pm$ 0.1 &98.9 $\pm$ 0.1 &99.0 $\pm$ 0.1 &96.3 $\pm$ 0.5 &98.0 $\pm$ 0.1\\
$\theta_b= bck_{2}$ &91.4 $\pm$ 1.0 &98.2 $\pm$ 0.0 &98.6 $\pm$ 0.1 &99.0 $\pm$ 0.1 &98.8 $\pm$ 0.1 &95.0 $\pm$ 0.1 &96.8 $\pm$ 0.1\\
$\theta_b= bck_{3}$ &88.7 $\pm$ 0.3 &97.5 $\pm$ 0.1 &98.5 $\pm$ 0.0 &98.9 $\pm$ 0.0 &98.6 $\pm$ 0.1 &94.3 $\pm$ 0.1 &96.1 $\pm$ 0.1\\
$\theta_b= bck_{4}$ &90.4 $\pm$ 0.6 &97.0 $\pm$ 0.3 &98.4 $\pm$ 0.0 &98.5 $\pm$ 0.1 &98.6 $\pm$ 0.1 &92.8 $\pm$ 0.1 &95.9 $\pm$ 0.2\\
\bottomrule
\end{tabular}}
\label{tab rmnist}
\end{table*}

\begin{table}[t]
\centering
\caption{Evaluations of ERM and the proposed basic framework with different settings of $\theta_b$ in the unseen domain from \cite{li2017deeper} with vision trasformer backbones (\ie DeiT \cite{touvron2021training} and T2T \cite{yuan2021tokens}). Here $bck_{1,2,3,4}$ is the 12 attention layers in them, with each block corresponding to 3 attention layers and $bck_4$ close to the classifier. Similar as the observation in the ResNet backbone, the common practice of using a shared base feature extractor with two lightweight prediction heads is detrimental to the performance, and the proposed early-branching structure leads to better results in vision transformers.}
\scalebox{0.85}{
\begin{tabular}{c|ccccc}
\toprule 
& Art & Cartoon & Photo & Sketch &Avg.\\
\hline \hline
\multicolumn{6}{c}{~~~~~~~~~~~~~~~~~~Deit backbone}\\
\hline
ERM-ViT &87.3 $\pm$ 0.4 &81.5 $\pm$ 0.3 &98.7 $\pm$ 0.4 &74.5 $\pm$ 1.9 &85.5 $\pm$ 0.6\\
$\theta_b=\textbf{I}$ &87.7 $\pm$ 0.6 &81.2 $\pm$ 0.4 &98.3 $\pm$ 0.2 &76.0 $\pm$ 2.5 &85.8 $\pm$ 0.4\\
$\theta_b= bck_{1}$   &88.4 $\pm$ 0.8 &80.5 $\pm$ 1.4 &98.8 $\pm$ 0.3 &77.8 $\pm$ 0.9 & 86.4 $\pm$ 0.2\\
$\theta_b= bck_{2}$ &86.9 $\pm$ 0.7 & 80.6 $\pm$ 0.9 &98.6 $\pm$ 0.0  &74.6 $\pm$ 1.2  &85.2 $\pm$ 0.4\\
$\theta_b= bck_{3}$ &86.5 $\pm$ 0.2 & 78.0 $\pm$ 0.8 & 98.4 $\pm$ 0.3 & 77.3 $\pm$ 0.3  &85.1 $\pm$ 0.3\\
$\theta_b= bck_{4}$ &83.2 $\pm$ 2.0 & 77.8 $\pm$ 1.1  & 97.9 $\pm$ 0.4 & 75.5 $\pm$ 0.5 &83.6 $\pm$ 0.5\\
\hline
\multicolumn{6}{c}{~~~~~~~~~~~~~~~~~~T2T backbone}\\
\hline
ERM-ViT &83.5 $\pm$ 0.2 &75.8 $\pm$ 0.7 &97.7 $\pm$ 0.3 &76.9 $\pm$ 0.8 &83.5 $\pm$ 0.3\\
$\theta_b=\textbf{I}$ &86.0 $\pm$ 0.9 &76.0 $\pm$ 2.2 &98.4 $\pm$ 0.1 &75.7 $\pm$ 1.8 &84.0 $\pm$ 1.0\\
$\theta_b= bck_{1}$   &85.1 $\pm$ 0.9 &75.6 $\pm$ 1.4 &98.4 $\pm$ 0.1 &77.3 $\pm$ 0.5 &84.1 $\pm$ 0.5\\
$\theta_b= bck_{2}$ &82.8 $\pm$ 0.9 &77.4 $\pm$ 1.6 &97.9 $\pm$ 0.5 &74.7 $\pm$ 1.5 &83.2 $\pm$ 0.3\\
$\theta_b= bck_{3}$ &81.6 $\pm$ 0.4 &78.4 $\pm$ 1.5 &97.5 $\pm$ 0.6 &75.4 $\pm$ 1.3 &83.2 $\pm$ 0.5\\
$\theta_b= bck_{4}$ &82.6 $\pm$ 0.6 &73.6 $\pm$ 0.8  &98.0 $\pm$ 0.2 &74.2 $\pm$ 1.2 &82.1 $\pm$ 0.3\\
\bottomrule
\end{tabular}}
\label{tab deit}
\end{table}

\begin{table*}[t]
\centering
\caption{Ablation studies regarding the effectivenesses of the proposed RDS augmentation strategy, the adopted losses, the adopted MLP layer for computing $\mathcal{L}_{cons}$, and the adopted HSIC independent constraint. Experiments are conducted in PACS \cite{li2017deeper} with the leave-one-out training-test strategy. Here RDS$^{-}$ is our RDS without Eq.~(\ref{eq sample}); $\mathcal{L}_{indp}^{oth}$ and $\mathcal{L}_{indp}^{corr}$ are variants of our method that use orthogonal and correlation as independent measurements.}
\scalebox{1}{
\begin{tabular}{llccccc}
\toprule 
& model & art & cartoon & photo & sketch &avg \\
\hline \hline
& Ours & 81.3  $\pm$  1.4 &75.7  $\pm$  0.2 &94.6  $\pm$  0.3 &78.1  $\pm$  0.7 &82.4  $\pm$  0.4 \\
\hline
\multirow{4}{*}{Augmentations} 
& Ours w/o augmentation &79.3  $\pm$  0.9 &74.1  $\pm$  2.0 &94.8  $\pm$  0.5 &75.8  $\pm$  1.6 &81.0  $\pm$  0.6\\
& Ours w/ MixStyle \cite{zhou2021domain} &80.8  $\pm$  0.8 &75.1  $\pm$  0.3 &94.4  $\pm$  0.5 &77.4  $\pm$  1.5 &81.9  $\pm$  0.3\\
&Ours w/ DSU \cite{li2022uncertainty} &81.3  $\pm$  0.5 &75.1  $\pm$  0.8 &94.2  $\pm$  0.4 &76.3  $\pm$  0.5 &81.7  $\pm$  0.3 \\
& Ours w/ RDS$^{-}$ &80.7  $\pm$  0.7 &75.4  $\pm$  0.5 &94.6  $\pm$  0.9 &76.6  $\pm$  1.7 &81.9  $\pm$  0.7\\
\hline
\multirow{3}{*}{Losses} 
&Ours w/o $\mathcal{L}_{indp}$ \& $\mathcal{L}_{cons}$ &79.9  $\pm$  0.7 &75.6  $\pm$  0.4 &94.2  $\pm$  0.2 &76.2  $\pm$  0.8 &81.5  $\pm$  0.3\\
&Ours w/o $\mathcal{L}_{indp}$ &80.5  $\pm$  0.5 &76.2  $\pm$  0.4 &94.0  $\pm$  0.2 &76.4  $\pm$  0.2 &81.8  $\pm$  0.1\\
&Ours w/o $\mathcal{L}_{cons}$ &79.3  $\pm$  0.4 &75.4  $\pm$  0.9 &94.5  $\pm$  0.2 &78.3  $\pm$  1.7 &81.9  $\pm$  0.5 \\
\hline
MLP layer
&Ours w/o MLP in $\mathcal{L}_{cons}$ &80.5 $\pm$ 0.4 &75.1 $\pm$ 0.6 &94.2 $\pm$ 0.3 &77.7 $\pm$ 1.3 &81.9 $\pm$ 0.4\\
\hline
\multirow{2}{*}{Measurements} 
&Ours w/ $\mathcal{L}_{indp}^{oth}$ &80.6  $\pm$  1.2 &75.4  $\pm$  0.5 &94.4  $\pm$  0.1 &77.9  $\pm$  1.4 &82.1  $\pm$  0.3\\
&Ours w/ $\mathcal{L}_{indp}^{corr}$ &81.2  $\pm$  0.6 &75.8  $\pm$  0.2 &94.2  $\pm$  0.6 &77.7  $\pm$  0.5 &82.2  $\pm$  0.3\\
\bottomrule
\end{tabular}}
\label{tab abl}
\end{table*}

\section{Analysis}
\label{sec ana} 
All experiments in this section are conducted on the widely-used PACS \cite{li2017deeper} benchmark with the rigorous evaluation settings illustrated in Sec. \ref{sec details}.

\subsection{Early-Branching Structure for Implementation}
\label{sec early}
When designing the dual-branch network (\ie $\theta_b \cup \theta_o \cup \theta_d$), the common practice is to use a deep sub-network as the shared base feature extractor and two lightweight prediction heads with much fewer model parameters \cite{ganin2016domain,albuquerque2019generalizing,atzmon2020causal,chen2021style}. According to proposition~\ref{prop}, this design will be detrimental to the performance because the causal feature extractor $F_o$ will rely heavily on the non-causal features $X_d$. We thus suggest an early-branching structure for the dual-branch network. To examine if the proposed early-branching structure can help improve DG, we conduct experiments by testing the performance with different branching locations.

We examine the performances with the classical ResNet18 and ResNet50 \cite{he2016deep} backbones that have four blocks. We try to alter the model architecture by choosing block $k\in\{1,2,3,4\}$ as the shared part. As for the corresponding two branches, we use $4-k$ blocks plus an adaptive pooling layer and a fully connected layer for feature generation (classifiers are attached at the end of each branch for final predictions). We also consider a special case when no shared base feature extractor is used, which is equivalent to using two separate networks for $F_o$ and $F_d$. We denote those cases as $\theta_b=\{\textbf{I}, bck_1, bck_{2},bck_{3},bck_{4}\}$ \footnote{In the context, $bck_{1, 2, 3, 4}$ is represented as $bck_{4}$ for simplicity, and similar for $bck_{1, 2, 3}$ (as $bck_{3}$) and $bck_{1, 2}$ (as $bck_{2}$).} and train the respective architecture with the classification loss and independent constraint (augmentation is not used for fair comparisons). Results shown in Table \ref{tab base_results} suggest that when the shared backbone is deep, i.e., $\theta_b = bck_{4}$, the performances are worse than the ERM baseline, and using an earlier branching structure, e.g., $\theta_b = bck_1$ leads to better performances, we thus use this setting for all our experiments.

To further verify our hypothesis, we conduct more evaluations with different network structures. \textbf{(1)} We evaluate the independent constraint on a simple network that consists of 4 convolution layers. Here the domain of $\theta_b$ (\ie $\theta_b=\{\textbf{I}, bck_1, bck_{2},bck_{3},bck_{4}\}$) is similar to that in the experiment with ResNet backbones with $bck_4$ closest to the classification head. We use the Rotated MNIST dataset \cite{ghifary2015domain} (\ie RMNIST) in this experiment. RMNIST is originally designed for the handwritten digit recognition task which consists of 70,000 samples and can be categorized into 6 domains by its rotation degree: $\{0, 15, 30, 45, 60, 75\}$.
\textbf{(2)} We adopt vision transformers to evaluate the early-branching structure. Specifically, we use the DeiT \cite{touvron2021training} and a T2T  \cite{yuan2021tokens} transformers that both with 12 attention layers in our evaluations. Considering the consistency with the aforementioned network structures, we divide the 12 attention layers into 4 blocks and each with 3 attention layers. In these two experiments, the PACS dataset \cite{li2017deeper} is utilized. 
Note in all three evaluations, the experimental settings are the same as that in the benchmark \cite{gulrajani2020search}, and we do not incorporate augmentation within the framework to ensure fair comparisons. 

Results are listed in Table~\ref{tab rmnist} and~\ref{tab deit}. We observe that we can obtain the optimal results on both the smaller (\ie 4 convs) convolution networks and vision transformers (\ie DeiT \cite{touvron2021training} and T2T \cite{yuan2021tokens}) with an early-branching structure, and the results decrease thereafter. All these results comply with our findings in proposition~\ref{prop}, and the proposed early-branching structure can serve as an effective complement to the independent-based framework regardless of different network backbones.

\def\swthree{0.243\linewidth}
\renewcommand{\tabcolsep}{1pt}
\begin{figure*}[t]
\centering
    \begin{tabular}{cccc}
        \includegraphics[width=\swthree]{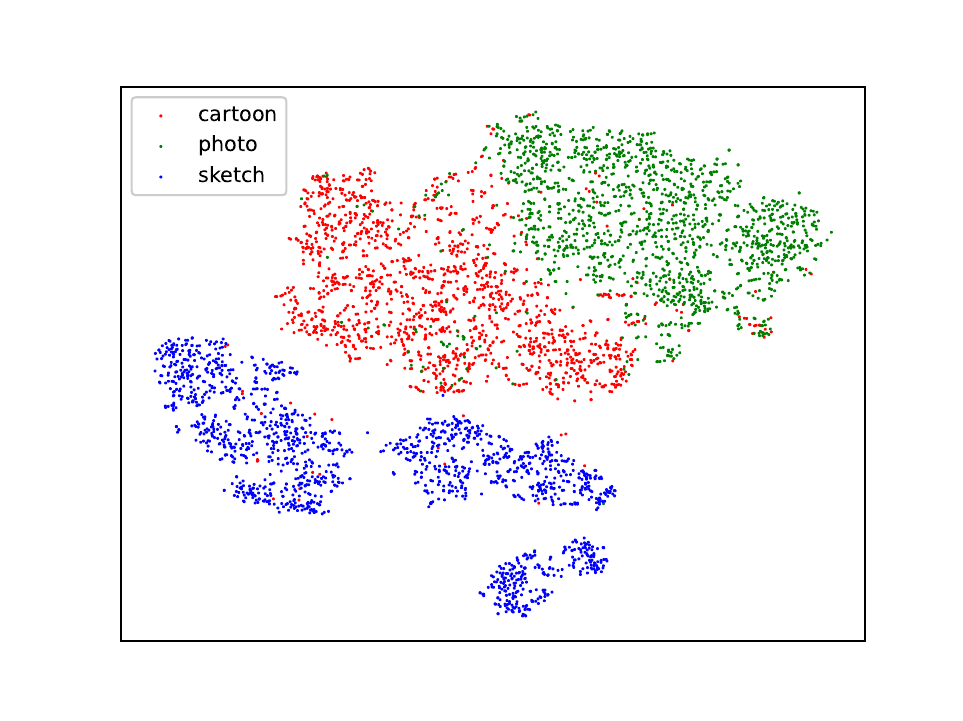}&
        \includegraphics[width=\swthree]{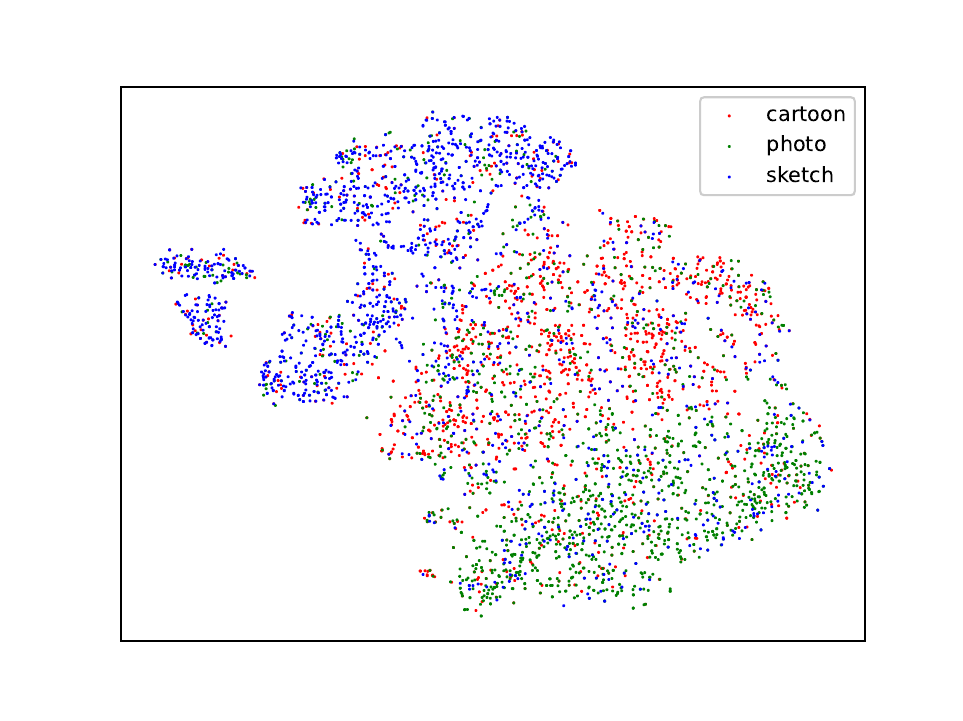}&
        \includegraphics[width=\swthree]{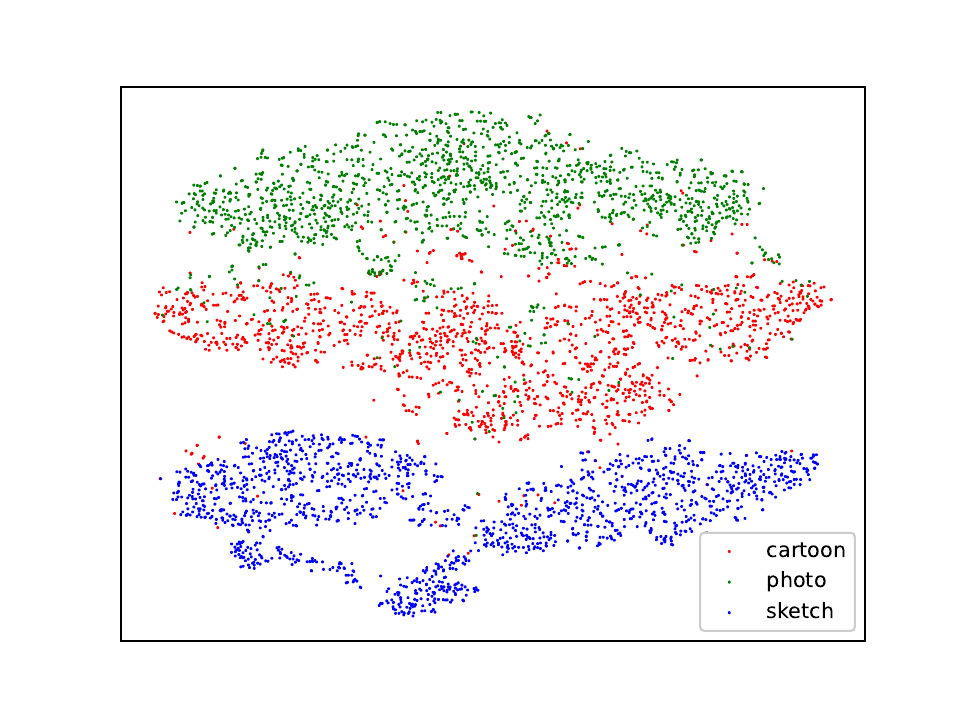}&
        \includegraphics[width=\swthree]{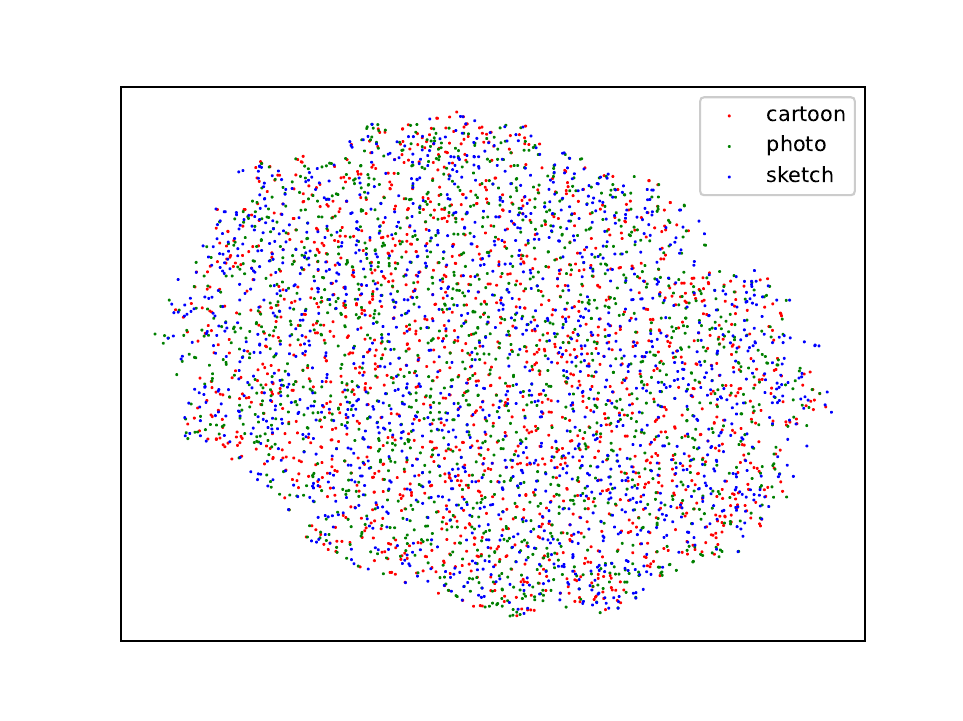}\\
        \footnotesize{\text{(a) ERM}} & \footnotesize{\text{(b) MixStyle}}& \footnotesize{\text{(c) DSU}} & \footnotesize{\text{(d) RDS}}\\
    \end{tabular}
	\caption{Visualizations of the original and augmented style statistics (concatenation of mean and variance). The PACS dataset \cite{li2017deeper} is used with art as the unseen domain during training. Augmented style statistics from MixStyle \cite{zhou2021domain} and DSU \cite{li2022uncertainty} can still reveal domain information, while our RDS can create diverse and inhomogeneous statistics.}
	\label{fig statistics}
\end{figure*}

\subsection{Effectiveness of the RDS Augmentation Strategy}
\label{sec aug}
Our augmentation strategy RDS is designed to alter the domain information of a sample while preserving its semantic information. We first compare our model against the baseline model that does not use any augmentation strategies (\ie Ours w/o augmentation in Table~\ref{tab abl}). 
As shown in the first row in Table~\ref{tab abl}, when enabling augmentation, our method can outperform the baseline model by a large margin in almost all domains except for "photo", which might be due to the ImageNet pretraining \cite{xu2020robust}. Note that the improvement for "sketch" is more pronounced than others, indicating our RDS can generalize to a target domain that significantly deviates from the source domains.

We also compare variants of our method by replacing RDS with other augmentations based on AdaIN (\ie Ours w/ MixStyle \cite{zhou2021domain} and Ours w/ DSU \cite{li2022uncertainty}). As listed in the 3rd and 4th rows, RDS consistently outperforms augmentations from \cite{zhou2021domain} and \cite{li2022uncertainty} on average. We postulate that this is due to RDS's capability of synthesizing more diverse domains thanks to the sampling scheme. 
Moreover, to test if RDS can bring improvements against the naive sampling strategy, we evaluate another variant of our method by disabling the process depicted in Eq.~(\ref{eq sample}) (\ie Ours w/ RDS$^{-}$). Results in the 5th row show that naive sampling is not as effective as the proposed sampling strategy, which is not surprising since styles from the naive sampling may not be diverse enough compared to that from RDS. Note that the improvements are more noticeable than others when it comes to the ``sketch" domain, which might be due to the reason that the rest three domains are close \cite{zhou2021domain}, thus requiring less diverse information for generalizing.
All these results validate the effectiveness of the proposed RDS against existing methods.

\subsection{Further Comparisons Between RDS and Others}
\label{sec moreaug}
This section further compares our RDS augmentation strategy with other two methods that based on AdaIN \cite{huang2017arbitrary} (\ie including MixStyle \cite{zhou2021domain} and DSU \cite{li2022uncertainty}). 
First, for MixStyle, given a batch of training sample $x$, the new style statistics $\hat{\mu}_{mix}$ and $\hat{\sigma}_{mix}$ can be obtained by mixing the statistics of the original sample and their corresponding shuffled one:
\begin{equation}
    \label{eq mixstyle}
    \hat{\mu}_{mix} = \lambda \mu + (1-\lambda) \tilde{\mu},~~\text{s.t.}~~\lambda\in\text{Beta}(\omega, \omega),~~\omega \in (0, \infty),
\end{equation}
where Beta($\cdot$) denotes the Beta distribution, $\mu$ and $\tilde{\mu}$ compute the mean of the original batch and its shuffled version. A similar method is applied to obtain another style statistic $\hat{\sigma}_{mix}$. 
Second, for DSU, their new statistics $\hat{\mu}_{dsu}$ and $\hat{\sigma}_{dsu}$ are obtained by adding perturbations to the original statistics:
\begin{equation}
\label{eq dsu}
\hat{\mu}_{dsu} = \mu + \phi \Sigma_{\mu},~~\text{s.t.}~~\phi \in \mathcal{N}(0,1),
\end{equation}
where $\Sigma_{\mu}$ is the variance of $\mu$ (similar to our definition in Eq.~(\ref{eq style})). The same approach is applied to obtain $\hat{\sigma}_{dsu}$.

To better illustrate the differences between their approaches and RDS, we plot the original and augmented style statistics in Figure~\ref{fig statistics}. As shown in Figure~\ref{fig statistics} (b) and (c), style statistics augmented from MixStyle and DSU can still reveal the domain information, indicating that their augmentations may create homogeneous features. In contrast, statistics from RDS are diverse and do not contain domain information, demonstrating that RDS obtains unlimited inhomogeneous features. These visualizations and the ablation studies in Table~\ref{tab abl} can validate the effectiveness of our RDS against other augmentation strategies.

\renewcommand{\tabcolsep}{7pt}
\begin{table*}[t]
\centering
\caption{Evaluations regarding different hyper-parameter (\ie $\alpha$ and $\beta$ in Eq.~(\ref{eq allloss})) settings. We fix one parameter and tune another when conducting the experiments which are examined in PACS \cite{li2017deeper} with the leave-one-out training-test strategy.}
\scalebox{1}{
\begin{tabular}{llccccc}
\toprule 
\multicolumn{2}{c}{hyper-parameters} & art & cartoon & photo & sketch &avg \\
\hline \hline
\multirow{4}{*}{$\beta=1$} & $\alpha=0.03$ &80.7 $\pm$ 0.7 &75.7 $\pm$ 0.9 &95.1 $\pm$ 0.2 &77.8 $\pm$ 0.7 &82.3 $\pm$ 0.4\\
&$\alpha=0.3$ &81.5 $\pm$ 0.1 &76.3 $\pm$ 1.2 &94.5 $\pm$ 0.2 &78.6 $\pm$ 1.0 &82.7 $\pm$ 0.4\\
&$\alpha=3$ &79.2 $\pm$ 0.8 &76.1 $\pm$ 0.8 &92.9 $\pm$ 0.1 &73.4 $\pm$ 0.8 &80.4 $\pm$ 0.3\\
&$\alpha=30$ &58.6 $\pm$ 6.9 &45.3 $\pm$ 9.0 &35.3 $\pm$ 5.7 &30.3 $\pm$ 4.5 &42.4 $\pm$ 5.9\\
\hline
\multirow{5}{*}{$\alpha=0.3$} &$\beta=0.1$ &82.6 $\pm$ 0.7 &76.3 $\pm$ 0.3 &94.2 $\pm$ 0.4 &74.9 $\pm$ 0.7 &82.0 $\pm$ 0.2 \\
&$\beta=1$ &81.5 $\pm$ 0.1 &76.3 $\pm$ 1.2 &94.5 $\pm$ 0.2 &78.6 $\pm$ 1.0 &82.7 $\pm$ 0.4\\
&$\beta=10$ &80.5 $\pm$ 0.6 &75.5 $\pm$ 0.2 &94.6 $\pm$ 0.3 &76.1 $\pm$ 0.9 &81.7 $\pm$ 0.4\\
&$\beta=100$ &81.4 $\pm$ 0.8 &76.8 $\pm$ 0.7 &93.1 $\pm$ 0.4 &75.9 $\pm$ 0.6 &81.8 $\pm$ 0.2\\
&$\beta=1000$ &53.1 $\pm$ 7.2 &58.7 $\pm$ 8.2 &63.3 $\pm$ 11.9 &62.6 $\pm$ 6.4 &59.4 $\pm$ 8.4\\
\bottomrule
\end{tabular}}
\label{tab params}
\vspace{-0.3cm}
\end{table*}

\subsection{Effectiveness of the Loss Terms}
In our complete algorithm, there are two additional loss terms in addition to the main classification losses: the conditional consistency loss $\mathcal{L}_{cons}$ and the independent constraint loss $\mathcal{L}_{indp}$. We study their effectiveness by disabling them separately in our framework. Results are listed in the 6th, 7th and 8th rows in Table~\ref{tab abl}. We observe that disabling either one of them can lead to performance decreases over our original design, demonstrating that each loss can independently contribute to the good performance of our method.
Meanwhile, we also note that when the two loss terms are both disabled, the corresponding model is inferior to other variants, meaning these two terms are complementary in our framework and can further improve the generalization.

\subsection{Effectiveness of the MLP layer}
We borrow the idea of using an additional MLP layer to prevent representation collapse from \cite{grill2020bootstrap}. Here we conduct an ablation study to examine if it is really necessary. To this end, we compare our method against a variant where the MLP layer is disabled in Eq.~(\ref{eq contrast}) (\ie Ours w/o MLP in $\mathcal{L}_{sen}$). As listed in the 9th row in Table~\ref{tab abl}, integrating the MLP layer can bring improvements in most test domains, validating its effectiveness in improving the DG performance.

\subsection{Effectiveness of the Adopted HSIC Independent Constraint}
\label{sec hsic}
We minimize the HSIC value between the causal and non-causal features (\ie $F_o(x)$ and $F_d(x)$) to reduce their dependency. As there exist other dependence measurements, such as orthogonal and correlation, one may wonder if HSIC is a better choice compared to them. To answer this question, we conduct ablation studies by replacing the HSIC measurement with the other two criteria. The orthogonal constraint requires the multiplication of the two features to be zero (\ie Ours w/ $\mathcal{L}_{indp}^{oth}$); the correlation constraint minimizes the cross-correlation matrix computed between $F_o(x)$ and $F_d(x)$, and we use the implementation from \cite{zbontar2021barlow} for this experiment (\ie Ours w/ $\mathcal{L}_{indp}^{corr}$). Experimental results listed in the 10th and 11th rows in Table~\ref{tab abl} suggest that the adopted HSIC independent constraint is more effective than the other two. We thus use the HSIC criterion in this work.

\subsection{Sensitive to the hyper-parameters.}
\label{sec params}
Our overall algorithm contains two hyper-parameters (\ie $\alpha$ and $\beta$ in Eq.~(\ref{eq allloss})), which are empirically fixed as $\alpha=0.3$ and $\beta=1$. To analyze the sensitivity of our model to these hyper-parameters, we conduct ablation studies by evaluating our model on PACS \cite{li2017deeper} using different settings of them. Note we fix the value for one hyper-parameter when analyzing another. Results are listed in Table~\ref{tab params}. We observe that our method performs consistency well when the hyper-parameters are set to $\alpha \in [0.03,0.3]$ and $\beta \in [0.1,1]$.

\subsection{Effects of the Additional Parameters}
\label{sec moreparam}
{Although additional parameters must be utilized in our framework for domain estimation, our method uses the same amount of parameters as ERM when deploying: the domain estimation branch employed during training will be discarded while testing.
To test if the extra parameters utilized during training can affect the performance, we adopt the same structure for ERM, which can be viewed as a variant of our method without augmentation and the independent constraint (i.e. ERM+). We test the variant in the PACS dataset using the same setting. The results in the second column of Table~\ref{add ana} show that this variant performs only on par with ERM, and is inferior to our method. From this, we can see that the improvements do not stem from the additional parameters utilized during training.}

\subsection{Combing ERM with Augmentations} 
\label{sec ermaug}
{We further combine ERM with different augmentation strategies to evaluate their effectiveness. Similarly, we use the same early-branching structure for the implementations, which can be viewed as our method with different augmentations but without the independent constraint. 
As shown in the 5th column in Table~\ref{add ana}, ERM performs superior to others when combined with the proposed RDS, which validates the effectiveness of our proposed augmentation scheme. Meanwhile, we note that ERM with RDS performs inferior to our method due to the lack of the independent constraint. This indicates the effectiveness of the adopted HSIC constraint to improve generalizations.}

\begin{table}[t]
\centering
\caption{{Evaluations of ERM with more parameters during training (i.e. ERM+) and with different augmentations.}}
\scalebox{0.78}{
\begin{tabular}{c|ccccc}
\toprule 
& Art & Cartoon & Photo & Sketch &Avg.\\
\hline \hline
ERM+ &80.6$\pm$0.7 & 72.9$\pm$1.0 &94.2$\pm$0.7 &73.1$\pm$1.4 &80.2$\pm$0.5\\
ERM+ w/ Mixstyle  &80.0$\pm$0.3 &75.1$\pm$0.7 &94.6$\pm$0.2 &75.3$\pm$1.3 &81.3$\pm$0.3\\
ERM+ w/ DSU &80.1$\pm$1.0 &74.4$\pm$1.8 &94.5$\pm$0.3 &75.3$\pm$0.8 &81.1$\pm$0.2\\
ERM+ w/ RDS &80.5$\pm$0.5 &76.2$\pm$0.4 &94.0$\pm$0.2 &76.4$\pm$0.2 &81.8$\pm$0.1\\
Ours &81.3$\pm$1.4 &75.7$\pm$0.2 &94.6$\pm$0.3 &78.1$\pm$0.7 &82.4$\pm$0.4\\
\bottomrule
\end{tabular}}
\label{add ana}
\end{table}

\subsection{Limitations and Future Work}
One limitation of this work is that both semantic and domain labels are required during training. Since we need the domain-specific features to partly guide the semantic representation learning, it is non-trivial to directly extend our method to occasions where domain labels are unavailable. We leave the exploration of such a setting to future work.
Another drawback is that the framework requires another thread for the domain classification task during training, bringing more parameters to the overall pipeline. Although we can use the same amount of parameters for the target classification task when deploying the model, training both two tasks will inevitably bring extra effort to the system (training time comparisons between different methods are provided in our supplementary material). A promising improvement for the method is to use much fewer parameters to obtain comparable performance against the original implementation.

\section{Conclusion}
This work proposes to ease the domain generalization problem from a causal perspective. The framework is based on the basic marginal independent assumption, where the causal semantic and non-causal domain-specific features are assumed to be independent, and it is implemented with a classical dual-branch network work where each type of feature is extracted through separate branches. Considering that the basic independent-based assumption may be ineffective in identifying the causal feature, we suggest two complements for the basic independent-based framework, which is built upon the conditional invariant observation.
First, we infer that the common practice of using a shared feature extractor with two lightweight prediction heads for the dual-branch network is not optimal. Instead, we suggest an early-branching network structure upon the classical dual-branch design, where the causal and non-causal feature obtaining branches share the first few blocks while diverging thereafter;
Second, we find that incorporating samples, which are with diverse domains that correspond to the same object, is essential for the framework and can further help the basic framework to characterize causal features. To this end, we introduce a carefully designed random domain sampling strategy for augmentation. We theoretically reveal that these two findings can further improve the learning outcome. Experiments on challenging benchmarks show that our complete algorithm can obtain competitive performance against current arts.

\section*{Data Availability Statements}
We use data from existing benchmark \cite{gulrajani2020search} for the training and test processes. They are available at \url{https://github.com/facebookresearch/DomainBed}. Detailed results from our method are provided in the appendix.

\bibliographystyle{plain}
{\footnotesize
\bibliography{dg.bib}}
\section*{Appendix: Detailed Results of the Benchmark Datasets}
\label{sec 3}
In this section, we present the average accuracy in each domain from different datasets. For all the experiments, we use the "training-domain validate set" as the model selection method, and all methods are trained with the leave-one-out strategy using the ResNet18 \cite{he2016deep} backbones. We also present the training times for different methods in the PACS datasets.

\begin{table*}[h]
\centering
\caption{Average accuracies on the PACS \cite{li2017deeper} datasets using the default hyper-parameter settings in DomainBed \cite{gulrajani2020search}. Time (min) here is the training time for one trial per GPU for the corresponding method.}
\scalebox{1}{
\begin{tabular}{lC{1.6cm}C{1.6cm}C{1.6cm}C{1.6cm}C{1.6cm}|C{0.8cm}}
\toprule 
& art & cartoon & photo & sketch & Average &Time\\
\hline \hline
ERM \cite{vapnik1999nature} &78.0 $\pm$ 1.3 &73.4 $\pm$ 0.8 &94.1 $\pm$ 0.4 &73.6 $\pm$ 2.2 &79.8 $\pm$ 0.4 &24\\
IRM \cite{arjovsky2019invariant} &76.9 $\pm$ 2.6 &75.1 $\pm$ 0.7 &94.3 $\pm$ 0.4 &77.4 $\pm$ 0.4 &80.9 $\pm$ 0.5 &18\\
GroupGRO \cite{sagawa2019distributionally} &77.7 $\pm$ 2.6 &76.4 $\pm$ 0.3 &94.0 $\pm$ 0.3 &74.8 $\pm$ 1.3 &80.7 $\pm$ 0.4 &24\\
Mixup \cite{yan2020improve} &79.3 $\pm$ 1.1 &74.2 $\pm$ 0.3 &94.9 $\pm$ 0.3 &68.3 $\pm$ 2.7 &79.2 $\pm$ 0.9 &18\\
MLDG \cite{li2018learning} &78.4 $\pm$ 0.7 &75.1 $\pm$ 0.5 &94.8 $\pm$ 0.4 &76.7 $\pm$ 0.8 &81.3 $\pm$ 0.2 &32\\
CORAL \cite{sun2016deep} &81.5 $\pm$ 0.5 &75.4 $\pm$ 0.7 &95.2 $\pm$ 0.5 &74.8 $\pm$ 0.4 &81.7 $\pm$ 0.0 &24\\
MMD \cite{li2018domain} &81.3 $\pm$ 0.6 &75.5 $\pm$ 1.0 &94.0 $\pm$ 0.5 &74.3 $\pm$ 1.5 &81.3 $\pm$ 0.8 &18\\
DANN \cite{ganin2016domain} &79.0 $\pm$ 0.6 &72.5 $\pm$ 0.7 &94.4 $\pm$ 0.5 &70.8 $\pm$ 3.0 &79.2 $\pm$ 0.3 &17\\
CDANN \cite{li2018deep} &80.4 $\pm$ 0.8 &73.7 $\pm$ 0.3 &93.1 $\pm$ 0.6 &74.2 $\pm$ 1.7 &80.3 $\pm$ 0.5 &24\\
MTL \cite{blanchard2021domain} &78.7 $\pm$ 0.6 &73.4 $\pm$ 1.0 &94.1 $\pm$ 0.6 &74.4 $\pm$ 3.0 &80.1 $\pm$ 0.8 &18\\
SagNet \cite{nam2021reducing} &82.9 $\pm$ 0.4 &73.2 $\pm$ 1.1 &94.6 $\pm$ 0.5 &76.1 $\pm$ 1.8 &81.7 $\pm$ 0.6 &24\\
ARM \cite{zhang2020adaptive} &79.4 $\pm$ 0.6 &75.0 $\pm$ 0.7 &94.3 $\pm$ 0.6 &73.8 $\pm$ 0.6 &80.6 $\pm$ 0.5 &24\\
VREx \cite {krueger2021out} &74.4 $\pm$ 0.7 &75.0 $\pm$ 0.4 &93.3 $\pm$ 0.3 &78.1 $\pm$ 0.9 &80.2 $\pm$ 0.5 &17\\
RSC \cite{huang2020self} &78.5 $\pm$ 1.1 &73.3 $\pm$ 0.9 &93.6 $\pm$ 0.6 &76.5 $\pm$ 1.4 &80.5 $\pm$ 0.2 &25\\
SelfReg \cite{kim2021selfreg} &82.5 $\pm$ 0.8 &74.4 $\pm$ 1.5 &95.4 $\pm$ 0.5 &74.9 $\pm$ 1.3 &81.8 $\pm$ 0.3 &25\\
MixStyle \cite{zhou2021domain} &82.6 $\pm$ 1.2 &76.3 $\pm$ 0.4 &94.2 $\pm$ 0.3 &77.5 $\pm$ 1.3 &82.6 $\pm$ 0.4 &25\\
Fish \cite{shi2021gradient} &80.9 $\pm$ 1.0 &75.9 $\pm$ 0.4 &95.0 $\pm$ 0.4 & 76.2 $\pm$ 1.0 &82.0 $\pm$ 0.3 &52\\
SD \cite{pezeshki2021gradient} &83.2 $\pm$ 0.6 &74.6 $\pm$ 0.3 &94.6 $\pm$ 0.1 &75.1 $\pm$ 1.6 &81.9 $\pm$ 0.3 &25\\
CAD \cite{ruan2021optimal} &83.9 $\pm$ 0.8 &74.2 $\pm$ 0.4 &94.6 $\pm$ 0.4 &75.0 $\pm$ 1.2 &81.9 $\pm$ 0.3 &24\\
CondCAD \cite{ruan2021optimal} &79.7 $\pm$ 1.0 &74.2 $\pm$ 0.9 &94.6 $\pm$ 0.4 &74.8 $\pm$ 1.4 &80.8 $\pm$ 0.5 &26\\
Fishr \cite{rame2021ishr} &81.2 $\pm$ 0.4 &75.8 $\pm$ 0.8 &94.3 $\pm$ 0.3 &73.8 $\pm$ 0.6 &81.3 $\pm$ 0.3 &17\\
Ours &81.3 $\pm$ 1.4 &75.7 $\pm$ 0.2 &94.6 $\pm$ 0.3 &78.1 $\pm$ 0.7 &82.4 $\pm$ 0.4 &29\\
\bottomrule
\end{tabular}}
\label{tab pacs}
\end{table*}

\begin{table*}[h]
\centering
\caption{Average accuracies on the VLCS \cite{fang2013unbiased} datasets using the default hyper-parameter settings in DomainBed \cite{gulrajani2020search}.}
\scalebox{1}{
\begin{tabular}{lC{1.8cm}C{1.8cm}C{1.8cm}C{1.8cm}C{1.8cm}}
\toprule 
&Caltech  & LabelMe &Sun & VOC & Average\\
\hline \hline
ERM \cite{vapnik1999nature} &97.7 $\pm$ 0.3 &62.1 $\pm$ 0.9 &70.3 $\pm$ 0.9 &73.2 $\pm$ 0.7 &75.8 $\pm$ 0.2 \\
IRM \cite{arjovsky2019invariant} &96.1 $\pm$ 0.8 &62.5 $\pm$ 0.3 &69.9 $\pm$ 0.7 &72.0 $\pm$ 1.4 &75.1 $\pm$ 0.1\\
GroupGRO \cite{sagawa2019distributionally} &96.7 $\pm$ 0.6 &61.7 $\pm$ 1.5 &70.2 $\pm$ 1.8 &72.9 $\pm$ 0.6 &75.4 $\pm$ 1.0 \\
Mixup \cite{yan2020improve} &95.6 $\pm$ 1.5 &62.7 $\pm$ 0.4 &71.3 $\pm$ 0.3 &75.4 $\pm$ 0.2 &76.2 $\pm$ 0.3 \\
MLDG \cite{li2018learning} &95.8 $\pm$ 0.5 &63.3 $\pm$ 0.8 &68.5 $\pm$ 0.5 &73.1 $\pm$ 0.8 &75.2 $\pm$ 0.3 \\
CORAL \cite{sun2016deep} &96.5 $\pm$ 0.3 &62.8 $\pm$ 0.1 &69.1 $\pm$ 0.6 &73.8 $\pm$ 1.0 &75.5 $\pm$ 0.4 \\
MMD \cite{li2018domain} &96.0 $\pm$ 0.8 &64.3 $\pm$ 0.6 &68.5 $\pm$ 0.6 &70.8 $\pm$ 0.1 &74.9 $\pm$ 0.5 \\
DANN \cite{ganin2016domain} &97.2 $\pm$ 0.1 &63.3 $\pm$ 0.6 &70.2 $\pm$ 0.9 &74.4 $\pm$ 0.2 &76.3 $\pm$ 0.2 \\
CDANN \cite{li2018deep} &95.4 $\pm$ 1.2 &62.6 $\pm$ 0.6 &69.9 $\pm$ 1.3 &76.2 $\pm$ 0.5 &76.0 $\pm$ 0.5 \\
MTL \cite{blanchard2021domain} &94.4 $\pm$ 2.3 &65.0 $\pm$ 0.6 &69.6 $\pm$ 0.6 &71.7 $\pm$ 1.3 &75.2 $\pm$ 0.3 \\
SagNet \cite{nam2021reducing} &94.9 $\pm$ 0.7 &61.9 $\pm$ 0.7 &69.6 $\pm$ 1.3 &75.2 $\pm$ 0.6 &75.4 $\pm$ 0.8 \\
ARM \cite{zhang2020adaptive} &96.9 $\pm$ 0.5 &61.9 $\pm$ 0.4 &71.6 $\pm$ 0.1 &73.3 $\pm$ 0.4 &75.9 $\pm$ 0.3 \\
VREx \cite {krueger2021out} &96.2 $\pm$ 0.0 &62.5 $\pm$ 1.3 &69.3 $\pm$ 0.9 &73.1 $\pm$ 1.2 &75.3 $\pm$ 0.6 \\
RSC \cite{huang2020self} &96.2 $\pm$ 0.0 &63.6 $\pm$ 1.3 &69.8 $\pm$ 1.0 &72.0 $\pm$ 0.4 &75.4 $\pm$ 0.3 \\
SelfReg \cite{kim2021selfreg} &95.8 $\pm$ 0.6 &63.4 $\pm$ 1.1 &71.1 $\pm$ 0.6 &75.3 $\pm$ 0.6 &76.4 $\pm$ 0.7 \\
MixStyle \cite{zhou2021domain} &97.3 $\pm$ 0.3 &61.6 $\pm$ 0.1 &70.4 $\pm$ 0.7 &71.3 $\pm$ 1.9 &75.2 $\pm$ 0.7 \\
Fish \cite{shi2021gradient} &97.4 $\pm$ 0.2 &63.4 $\pm$ 0.1 &71.5 $\pm$ 0.4 &75.2 $\pm$ 0.7 &76.9 $\pm$ 0.2 \\
SD \cite{pezeshki2021gradient} &96.5 $\pm$ 0.4 &62.2 $\pm$ 0.0 &69.7 $\pm$ 0.9 &73.6 $\pm$ 0.4 &75.5 $\pm$ 0.4 \\
CAD \cite{ruan2021optimal} &94.5 $\pm$ 0.9 &63.5 $\pm$ 0.6 &70.4 $\pm$ 1.2 &72.4 $\pm$ 1.3 &75.2 $\pm$ 0.6 \\
CondCAD \cite{ruan2021optimal} &96.5 $\pm$ 0.8 &62.6 $\pm$ 0.4 &69.1 $\pm$ 0.2 &76.0 $\pm$ 0.2 &76.1 $\pm$ 0.3 \\
Fishr \cite{rame2021ishr} &97.2 $\pm$ 0.6 &63.3 $\pm$ 0.7 &70.4 $\pm$ 0.6 &74.0 $\pm$ 0.8 &76.2 $\pm$ 0.3\\
Ours &97.0 $\pm$ 0.2 &63.7 $\pm$ 0.9 &71.0 $\pm$ 0.3 &74.4 $\pm$ 0.8 &76.5 $\pm$ 0.4 \\
\bottomrule
\end{tabular}}
\label{tab vlcs}
\end{table*}

\begin{table*}
\centering
\caption{Average accuracies on the OfficeHome \cite{venkateswara2017deep} datasets using the default hyper-parameter settings in DomainBed \cite{gulrajani2020search}.}
\scalebox{1}{
\begin{tabular}{lC{1.8cm}C{1.8cm}C{1.8cm}C{1.8cm}C{1.8cm}}
\toprule 
&art  & clipart &product & real & Average\\
\hline \hline
ERM \cite{vapnik1999nature} &52.2 $\pm$ 0.2 &48.7 $\pm$ 0.5 &69.9 $\pm$ 0.5 &71.7 $\pm$ 0.5 &60.6 $\pm$ 0.2 \\
IRM \cite{arjovsky2019invariant} &49.7 $\pm$ 0.2 &46.8 $\pm$ 0.5 &67.5 $\pm$ 0.4 &68.1 $\pm$ 0.6 &58.0 $\pm$ 0.1\\
GroupGRO \cite{sagawa2019distributionally} &52.6 $\pm$ 1.1 &48.2 $\pm$ 0.9 &69.9 $\pm$ 0.4 &71.5 $\pm$ 0.8 &60.6 $\pm$ 0.3 \\
Mixup \cite{yan2020improve} &54.0 $\pm$ 0.7 &49.3 $\pm$ 0.7 &70.7 $\pm$ 0.7 &72.6 $\pm$ 0.3 &61.7 $\pm$ 0.5 \\
MLDG \cite{li2018learning} &53.1 $\pm$ 0.3 &48.4 $\pm$ 0.3 &70.5 $\pm$ 0.7 &71.7 $\pm$ 0.4 &60.9 $\pm$ 0.2 \\
CORAL \cite{sun2016deep} &55.1 $\pm$ 0.7 &49.7 $\pm$ 0.9 &71.8 $\pm$ 0.2 &73.1 $\pm$ 0.5 &62.4 $\pm$ 0.4 \\
MMD \cite{li2018domain} &50.9 $\pm$ 1.0 &48.7 $\pm$ 0.3 &69.3 $\pm$ 0.7 &70.7 $\pm$ 1.3 &59.9 $\pm$ 0.4 \\
DANN \cite{ganin2016domain} &51.8 $\pm$ 0.5 &47.1 $\pm$ 0.1 &69.1 $\pm$ 0.7 &70.2 $\pm$ 0.7 &59.5 $\pm$ 0.5 \\
CDANN \cite{li2018deep} &51.4 $\pm$ 0.5 &46.9 $\pm$ 0.6 &68.4 $\pm$ 0.5 &70.4 $\pm$ 0.4 &59.3 $\pm$ 0.4 \\
MTL \cite{blanchard2021domain} &51.6 $\pm$ 1.5 &47.7 $\pm$ 0.5 &69.1 $\pm$ 0.3 &71.0 $\pm$ 0.6 &59.9 $\pm$ 0.5 \\
SagNet \cite{nam2021reducing} &55.3 $\pm$ 0.4 &49.6 $\pm$ 0.2 &72.1 $\pm$ 0.4 &73.2 $\pm$ 0.4 &62.5 $\pm$ 0.3 \\
ARM \cite{zhang2020adaptive} &51.3 $\pm$ 0.9 &48.5 $\pm$ 0.4 &68.0 $\pm$ 0.3 &70.6 $\pm$ 0.1 &59.6 $\pm$ 0.3 \\
VREx \cite {krueger2021out} &51.1 $\pm$ 0.3 &47.4 $\pm$ 0.6 &69.0 $\pm$ 0.4 &70.5 $\pm$ 0.4 &59.5 $\pm$ 0.1 \\
RSC \cite{huang2020self} &49.0 $\pm$ 0.1 &46.2 $\pm$ 1.5 &67.8 $\pm$ 0.7 &70.6 $\pm$ 0.3 &58.4 $\pm$ 0.6 \\
SelfReg \cite{kim2021selfreg} &55.1 $\pm$ 0.8 &49.2 $\pm$ 0.6 &72.2 $\pm$ 0.3 &73.0 $\pm$ 0.3 &62.4 $\pm$ 0.1 \\
MixStyle \cite{zhou2021domain} &50.8 $\pm$ 0.6 &51.4 $\pm$ 1.1 &67.6 $\pm$ 1.3 &68.8 $\pm$ 0.5 &59.6 $\pm$ 0.8 \\
Fish \cite{shi2021gradient} &54.6 $\pm$ 1.0 &49.6 $\pm$ 1.0 &71.3 $\pm$ 0.6 &72.4 $\pm$ 0.2 &62.0 $\pm$ 0.6 \\
SD \cite{pezeshki2021gradient} &55.0 $\pm$ 0.4 &51.3 $\pm$ 0.5 &72.5 $\pm$ 0.2 &72.7 $\pm$ 0.3 &62.9 $\pm$ 0.2 \\
CAD \cite{ruan2021optimal} &52.1 $\pm$ 0.6 &48.3 $\pm$ 0.5 &69.7 $\pm$ 0.3 &71.9 $\pm$ 0.4 &60.5 $\pm$ 0.3 \\
CondCAD \cite{ruan2021optimal} &53.3 $\pm$ 0.6 &48.4 $\pm$ 0.2 &69.8 $\pm$ 0.9 &72.6 $\pm$ 0.1 &61.0 $\pm$ 0.4 \\
Fishr \cite{rame2021ishr} &52.6 $\pm$ 0.9 &48.6 $\pm$ 0.3 &69.9 $\pm$ 0.6 &72.4 $\pm$ 0.4 &60.9 $\pm$ 0.3 \\
Ours &54.5 $\pm$ 0.3 &52.4 $\pm$ 0.2 &71.4 $\pm$ 0.5 &70.3 $\pm$ 0.3 &62.2 $\pm$ 0.1 \\
\bottomrule
\end{tabular}}
\label{tab office}
\end{table*}

\newpage
\begin{table*}
\centering
\caption{Average accuracies on the TerraInc \cite{beery2018recognition} datasets using the default hyper-parameter settings in DomainBed \cite{gulrajani2020search}.}
\scalebox{1}{
\begin{tabular}{lC{1.8cm}C{1.8cm}C{1.8cm}C{1.8cm}C{1.8cm}}
\toprule 
&L100  & L38 &L43 & L46 & Average\\
\hline \hline
ERM \cite{vapnik1999nature} &42.1 $\pm$ 2.5 &30.1 $\pm$ 1.2 &48.9 $\pm$ 0.6 &34.0 $\pm$ 1.1 &38.8 $\pm$ 1.0 \\
IRM \cite{arjovsky2019invariant} &41.8 $\pm$ 1.8 &29.0 $\pm$ 3.6 &49.6 $\pm$ 2.1 &33.1 $\pm$ 1.5 &38.4 $\pm$ 0.9\\
GroupGRO \cite{sagawa2019distributionally} &45.3 $\pm$ 4.6 &36.1 $\pm$ 4.4 &51.0 $\pm$ 0.8 &33.7 $\pm$ 0.9 &41.5 $\pm$ 2.0 \\
Mixup \cite{yan2020improve} &49.4 $\pm$ 2.0 &35.9 $\pm$ 1.8 &53.0 $\pm$ 0.7          &30.0 $\pm$ 0.9 &42.1 $\pm$ 0.7 \\
MLDG \cite{li2018learning} &39.6 $\pm$ 2.3 &33.2 $\pm$ 2.7 &52.4 $\pm$ 0.5          &35.1 $\pm$ 1.5 &40.1 $\pm$ 0.9 \\
CORAL \cite{sun2016deep} &46.7 $\pm$ 3.2 &36.9 $\pm$ 4.3 &49.5 $\pm$ 1.9          &32.5 $\pm$ 0.7 &41.4 $\pm$ 1.8 \\
MMD \cite{li2018domain} &49.1 $\pm$ 1.2 &36.4 $\pm$ 4.8 &50.4 $\pm$ 2.1          &32.3 $\pm$ 1.5 &42.0 $\pm$ 1.0 \\
DANN \cite{ganin2016domain} &44.3 $\pm$ 3.6 &28.0 $\pm$ 1.5 &47.9 $\pm$ 1.0          &31.3 $\pm$ 0.6 &37.9 $\pm$ 0.9 \\
CDANN \cite{li2018deep} &36.9 $\pm$ 6.4 &32.7 $\pm$ 6.2 &51.1 $\pm$ 1.3          &33.5 $\pm$ 0.5 &38.6 $\pm$ 2.3 \\
MTL \cite{blanchard2021domain} &45.2 $\pm$ 2.6 &31.0 $\pm$ 1.6 &50.6 $\pm$ 1.1          &34.9 $\pm$ 0.4 &40.4 $\pm$ 1.0 \\
SagNet \cite{nam2021reducing} &36.3 $\pm$ 4.7 &40.3 $\pm$ 2.0 &52.5 $\pm$ 0.6          &33.3 $\pm$ 1.3 &40.6 $\pm$ 1.5 \\
ARM \cite{zhang2020adaptive} &41.5 $\pm$ 4.5 &27.7 $\pm$ 2.4 &50.9 $\pm$ 1.0          &29.6 $\pm$ 1.5 &37.4 $\pm$ 1.9 \\
VREx \cite {krueger2021out} &48.0 $\pm$ 1.7 &41.1 $\pm$ 1.5 &51.8 $\pm$ 1.5          &32.0 $\pm$ 1.2 &43.2 $\pm$ 0.3 \\
RSC \cite{huang2020self} &42.8 $\pm$ 2.4 &32.2 $\pm$ 3.8 &49.6 $\pm$ 0.9          &32.9 $\pm$ 1.2 &39.4 $\pm$ 1.3 \\
SelfReg \cite{kim2021selfreg} &46.1 $\pm$ 1.5 &34.5 $\pm$ 1.6 &49.8 $\pm$ 0.3          &34.7 $\pm$ 1.5 &41.3 $\pm$ 0.3 \\
MixStyle \cite{zhou2021domain} &50.6 $\pm$ 1.9 &28.0 $\pm$ 4.5 &52.1 $\pm$ 0.7          &33.0 $\pm$ 0.2 &40.9 $\pm$ 1.1 \\
Fish \cite{shi2021gradient} &46.3 $\pm$ 3.0 &29.0 $\pm$ 1.1 &52.7 $\pm$ 1.2          &32.8 $\pm$ 1.0 &40.2 $\pm$ 0.6 \\
SD \cite{pezeshki2021gradient} &45.5 $\pm$ 1.9 &33.2 $\pm$ 3.1 &52.9 $\pm$ 0.7          &36.4 $\pm$ 0.8 &42.0 $\pm$ 1.0 \\
CAD \cite{ruan2021optimal} &43.1 $\pm$ 2.6 &31.1 $\pm$ 1.9 &53.1 $\pm$ 1.6          &34.7 $\pm$ 1.3 &40.5 $\pm$ 0.4 \\
CondCAD \cite{ruan2021optimal} &44.4 $\pm$ 2.9 &32.9 $\pm$ 2.5 &50.5 $\pm$ 1.3          &30.8 $\pm$ 0.5 &39.7 $\pm$ 0.4 \\
Fishr \cite{rame2021ishr} &49.9 $\pm$ 3.3 &36.6 $\pm$ 0.9 &49.8 $\pm$ 0.2 &34.2 $\pm$ 1.3 &42.6 $\pm$ 1.0 \\
Ours &50.8 $\pm$ 4.5 &35.8 $\pm$ 0.9 &51.1 $\pm$ 1.3 &35.2 $\pm$ 2.6 &43.2 $\pm$ 1.3 \\
\bottomrule
\end{tabular}}
\label{tab terainc}
\end{table*}

\begin{table*}
\centering
\caption{Average accuracies on the DomainNet \cite{peng2019moment} datasets using the default hyper-parameter settings in DomainBed \cite{gulrajani2020search}.}
\scalebox{1}{
\begin{tabular}{lC{1.7cm}C{1.5cm}C{1.7cm}C{1.5cm}C{1.7cm}C{1.6cm}c}
\toprule 
&clip  & info &paint & quick &real &sketch & Average\\
\hline \hline
ERM \cite{vapnik1999nature} &50.4 $\pm$ 0.2 &14.0 $\pm$ 0.2 &40.3 $\pm$ 0.5          &11.7 $\pm$ 0.2 &52.0 $\pm$ 0.2 &43.2 $\pm$ 0.3 &35.3 $\pm$ 0.1 \\
IRM \cite{arjovsky2019invariant} &43.2 $\pm$ 0.9 &12.6 $\pm$ 0.3 &35.0 $\pm$ 1.4          &9.9 $\pm$ 0.4 &43.4 $\pm$ 3.0 &38.4 $\pm$ 0.4 &30.4 $\pm$ 1.0\\
GroupGRO \cite{sagawa2019distributionally} &38.2 $\pm$ 0.5 &13.0 $\pm$ 0.3          &28.7 $\pm$ 0.3 &8.2 $\pm$ 0.1 &43.4 $\pm$ 0.5 &33.7 $\pm$ 0.0 &27.5 $\pm$ 0.1 \\
Mixup \cite{yan2020improve} &48.9 $\pm$ 0.3 &13.6 $\pm$ 0.3 &39.5 $\pm$ 0.5          &10.9 $\pm$ 0.4 &49.9 $\pm$ 0.2 &41.2 $\pm$ 0.2 &34.0 $\pm$ 0.0 \\
MLDG \cite{li2018learning} &51.1 $\pm$ 0.3 &14.1 $\pm$ 0.3 &40.7 $\pm$ 0.3          &11.7 $\pm$ 0.1 &52.3 $\pm$ 0.3 &42.7 $\pm$ 0.2 &35.4 $\pm$ 0.0 \\
CORAL \cite{sun2016deep} &51.2 $\pm$ 0.2 &15.4 $\pm$ 0.2 &42.0 $\pm$ 0.2 &12.7 $\pm$ 0.1 &52.0 $\pm$ 0.3 &43.4 $\pm$ 0.0 &36.1 $\pm$ 0.2 \\
MMD \cite{li2018domain} &16.6 $\pm$ 13.3 &0.3 $\pm$ 0.0 &12.8 $\pm$ 10.4 &0.3 $\pm$ 0.0 &17.1 $\pm$ 13.7 &0.4 $\pm$ 0.0 &7.9 $\pm$ 6.2 \\
DANN \cite{ganin2016domain}  &45.0 $\pm$ 0.2 &12.8 $\pm$ 0.2 &36.0 $\pm$ 0.2          &10.4 $\pm$ 0.3 &46.7 $\pm$ 0.3 &38.0 $\pm$ 0.3 &31.5 $\pm$ 0.1\\
CDANN \cite{li2018deep} &45.3 $\pm$ 0.2 &12.6 $\pm$ 0.2 &36.6 $\pm$ 0.2          &10.3 $\pm$ 0.4 &47.5 $\pm$ 0.1 &38.9 $\pm$ 0.4 &31.8 $\pm$ 0.2 \\
MTL \cite{blanchard2021domain} &50.6 $\pm$ 0.2 &14.0 $\pm$ 0.4 &39.6 $\pm$ 0.3          &12.0 $\pm$ 0.3 &52.1 $\pm$ 0.1 &41.5 $\pm$ 0.0 &35.0 $\pm$ 0.0 \\
SagNet \cite{nam2021reducing} &51.0 $\pm$ 0.1 &14.6 $\pm$ 0.1 &40.2 $\pm$ 0.2          &12.1 $\pm$ 0.2 &51.5 $\pm$ 0.3 &42.4 $\pm$ 0.1 &35.3 $\pm$ 0.1 \\
ARM \cite{zhang2020adaptive} &43.0 $\pm$ 0.2 &11.7 $\pm$ 0.2 &34.6 $\pm$ 0.1          &9.8 $\pm$ 0.4 &43.2 $\pm$ 0.3 &37.0 $\pm$ 0.3 &29.9 $\pm$ 0.1 \\
VREx \cite {krueger2021out} &39.2 $\pm$ 1.6 &11.9 $\pm$ 0.4 &31.2 $\pm$ 1.3          &10.2 $\pm$ 0.4 &41.5 $\pm$ 1.8 &34.8 $\pm$ 0.8 &28.1 $\pm$ 1.0 \\
RSC \cite{huang2020self} &39.5 $\pm$ 3.7 &11.4 $\pm$ 0.8 &30.5 $\pm$ 3.1          &10.2 $\pm$ 0.8 &41.0 $\pm$ 1.4 &34.7 $\pm$ 2.6 &27.9 $\pm$ 2.0 \\
SelfReg \cite{kim2021selfreg} &47.9 $\pm$ 0.3 &15.1 $\pm$ 0.3 &41.2 $\pm$ 0.2          &11.7 $\pm$ 0.3 &48.8 $\pm$ 0.0 &43.8 $\pm$ 0.3 &34.7 $\pm$ 0.2 \\
MixStyle \cite{zhou2021domain} &49.1 $\pm$ 0.4 &13.4 $\pm$ 0.0 &39.3 $\pm$ 0.0 &11.4 $\pm$ 0.4 &47.7 $\pm$ 0.3 &42.7 $\pm$ 0.1 &33.9 $\pm$ 0.1\\
Fish \cite{shi2021gradient} &51.5 $\pm$ 0.3 &14.5 $\pm$ 0.2 &40.4 $\pm$ 0.3          &11.7 $\pm$ 0.5 &52.6 $\pm$ 0.2 &42.1 $\pm$ 0.1 &35.5 $\pm$ 0.0 \\
SD \cite{pezeshki2021gradient} &51.3 $\pm$ 0.3 &15.5 $\pm$ 0.1 &41.5 $\pm$ 0.3          &12.6 $\pm$ 0.2 &52.9 $\pm$ 0.2 &44.0 $\pm$ 0.4 &36.3 $\pm$ 0.2 \\
CAD \cite{ruan2021optimal} &45.4 $\pm$ 1.0 &12.1 $\pm$ 0.5 &34.9 $\pm$ 1.1          &10.2 $\pm$ 0.6 &45.1 $\pm$ 1.6 &38.5 $\pm$ 0.6 &31.0 $\pm$ 0.8 \\
CondCAD \cite{ruan2021optimal} &46.1 $\pm$ 1.0 &13.3 $\pm$ 0.4 &36.1 $\pm$ 1.4          &10.7 $\pm$ 0.2 &46.8 $\pm$ 1.3 &38.7 $\pm$ 0.7 &31.9 $\pm$ 0.7 \\
Fishr \cite{rame2021ishr} &47.8 $\pm$ 0.7 &14.6 $\pm$ 0.2 &40.0 $\pm$ 0.3 &11.9 $\pm$ 0.2 &49.2 $\pm$ 0.7 &41.7 $\pm$ 0.1 &34.2 $\pm$ 0.3 \\
Ours &48.8 $\pm$ 0.7 &14.7 $\pm$ 0.2 &40.8 $\pm$ 0.1 &11.4 $\pm$ 0.2 &49.1 $\pm$ 0.1 &44.6 $\pm$ 0.0 &34.9 $\pm$ 0.1 \\
\bottomrule
\end{tabular}}
\label{tab domainnet}
\end{table*}

\clearpage
\newpage
\appendix
\onecolumn

\end{document}